\documentclass[journal]{IEEEtran}
\usepackage[english]{babel}
\usepackage{bm}
\ifCLASSINFOpdf
\else
\fi
\usepackage{amsmath, amsmath, amssymb, amsthm, graphicx}
\usepackage[numbers]{natbib} 
\newtheorem{theorem}{Theorem}
\newtheorem{prop}{Proposition}
\newtheorem{lem}{Lemma}
\newtheorem{defn}{Definition}
\newtheorem{cor}{Corollary}
\newtheorem{fact}{Fact}
\newtheorem{example}{Example}

\newtheoremstyle{PropositionNum}
{\topsep}{\topsep}              
{\itshape}                      
{}                              
{\bfseries}                     
{.}                             
{ }                             
{\thmname{#1}\thmnote{ \bfseries #3}}
\theoremstyle{PropositionNum}
\newtheorem{propn}{Proposition}

\begin{document}
%
\title{How much does your data exploration overfit? \\Controlling bias via information usage.}
%
%
%

\author{Daniel Russo and James Zou
\thanks{Daniel Russo is with the division of Decision, Risk and Operations at Columbia University.}
\thanks{James Zou is with Biomedical Data Science and, by courtesy, of Computer Science and Electrical Engineering at Stanford University}
\thanks{Manuscript received January 30, 2017; revised May 30, 2018; accepted September 12, 2019}
\thanks{Part of this work was presented at AISTATS 2016.}}

\newcommand{\E}{\mathbf{E}}
\newcommand{\N}{\mathcal{N}}
\newcommand{\Prob}{\mathbf{P}}
\newcommand{\phibf}{\bm \phi}
\newcommand{\mubf}{\bm \mu}
\newcommand{\Ybf}{\bm Y}
\newcommand{\Hc}{\mathcal{H}}

\maketitle

\begin{abstract}
	Modern data is messy and high-dimensional, and it is often not clear a priori what are the right questions to ask. Instead, the analyst typically needs to use the data to search for interesting analyses to perform and hypotheses to test. This is an adaptive process, where the choice of analysis to be performed next depends on the results of the previous analyses on the same data. Ultimately, which results are reported can be heavily influenced by the data. It is widely recognized that this process, even if well-intentioned, can lead to biases and false discoveries, contributing to the crisis of reproducibility in science.  But while any data-exploration renders standard statistical theory invalid,  experience suggests that different types of exploratory analysis can lead to disparate levels of bias, and the degree of bias also depends on the particulars of the data set.  In this paper, we propose a general information usage framework to quantify and provably bound the bias and other error metrics of an arbitrary exploratory analysis. We prove that our mutual information based bound is tight in natural settings, and then use it to give rigorous insights into when commonly used procedures do or do not lead to substantially biased estimation. Through the lens of information usage, we analyze the bias of  specific exploration procedures such as filtering, rank selection and clustering. Our general framework also naturally motivates randomization techniques that provably reduce exploration bias while preserving the utility of the data analysis. We discuss the connections between our approach and related ideas from differential privacy and blinded data analysis, and supplement our results with illustrative simulations. 
\end{abstract}


%
\IEEEpeerreviewmaketitle

\section{Introduction}
%
%
%
%
\IEEEPARstart{M}{odern} data is messy and high dimensional, and it is often not clear a priori what is the right analysis to perform. To extract the most insight, the analyst typically needs to perform exploratory analysis to make sense of the data and identify interesting hypotheses. This is invariably an adaptive process; patterns in the data observed in the first stages of analysis inform which tests are run next and the process iterates. Ultimately, the data itself may influence which results the analyst chooses to report,  introducing \emph{researcher degrees of freedom}: an additional source of over-fitting that isn't accounted for in reported statistical estimates \cite{simmons2011false}. Even if the analyst is well-intentioned, this exploration can lead to false discovery or large bias in reported estimates.

The practice of data-exploration is largely outside the domain of classical statistical theory. Standard tools of multiple hypothesis testing and false discovery rate (FDR) control assume that all the hypotheses to be tested, and the procedure for testing them, are chosen independently of the dataset. Any ``peeking'' at the data before committing to an analysis procedure renders classical statistical theory invalid. 
Nevertheless, data exploration is ubiquitous, and folklore and experience 
suggest the risk of false discoveries differs substantially depending on how the analyst explores the data. This creates a glaring gap between the messy practice of data analysis, and the standard theoretical frameworks used to understand statistical procedures. In this paper, we aim to narrow this gap.  We develop a general framework based on the concept of information usage and systematically study the degree of bias introduced by different forms of exploratory analysis, in which the choice of which function of the data to report is made {\it after} observing and analyzing the dataset.

To concretely illustrate the challenges of data exploration, consider two data scientists Alice and Bob.
 
{\bf Example 1.} \emph{Alice has a dataset of 1000 individuals for a weight-loss biomarker study. For each individual, she has their weight measured at 3 time points and the current expression values of 2000 genes assayed from blood samples. There are three possible weight changes that Alice could have looked at---the difference between time points 1 and 2, 2 and 3 or 1 and 3---but Alice decides ahead of time to only analyze the weight change between 1 and 3. She computes the correlation across individuals between the expression of each gene and the weight change, and reports the gene with the highest correlations along with its $r^2$ value. This is a canonical setting where we have tools for controlling error in multiple-hypothesis testing and the false-discovery rate (FDR). It is well-recognized that even if the reported gene passes the multiple-testing threshold, its correlation in independent replication studies tend to be smaller than the reported correlation in the current study. This phenomenon is also called the Winner's Curse selection bias.}

{\bf Example 2.} \emph{Bob has the same data, and he performs some simple data exploration. He first uses data visualization to investigate the average expression of all the genes across all the individuals at each of the time points, and observes that there is very little difference between time 1 and 2 and there is a large jump between time 2 and 3 in the average expression. So he decides to focus on these later two time points. Next, he realizes that half of the genes always have low expression values and decides to simply filter them out. Finally, he computes the correlations between the expression of the 1000 post-filtered genes and the weight change between time 2 and 3. He selects the gene with the largest correlation and reports its value. Bob's analysis consists of three steps and the results of each step depend on the results and decisions made in the previous steps. This adaptivity in Bob's exploration makes it difficult to apply standard statistical frameworks. We suspect there is also a selection bias here leading to the reported correlation being systematically larger than the real correlations if those genes are tested again. How do we think about and quantify the selection bias and overfitting due to this more complex data exploration? When is it larger or smaller than Alice's selection bias?}

The toy examples of Alice and Bob illustrate several subtleties of bias due to data exploration. First, the adaptivity of Bob's analysis makes it more difficult to quantify its bias compared to Alice's analysis. Second, for the same analysis procedure, the amount of selection bias depends on the dataset. Take Alice for example, if across the population one gene  is substantially more correlated with weight change than all other genes, then we expect the magnitude of Winner's Curse decreases. Third, different steps of data exploration introduce different amounts of selection bias. Intuitively, Bob's visualizing of aggregate expression values in the beginning should not introduce as much selection bias as his selection of the top gene at the last step. 

This paper introduces a mathematical framework to formalize these intuitions and to study selection bias from data exploration. The main tool we develop is a metric of the \emph{bad} information usage in the data exploration. The true signal in a dataset is the signal that is preserved in a replication dataset, and the noise is what changes across different replications. Using Shannon's mutual information, we quantify the degree of dependence between the noise in the data and the choice of which result is reported. We then prove that the bias of an arbitrary data-exploration process is bounded by this measure of its bad information usage. 
This bound provides a quantitative measure of researcher degrees of freedom, and offers a single lens through which we investigate different forms of exploration. 

In Section~\ref{sec:model}, we present a general model of exploratory data-analysis that encompasses the procedures used by Alice and Bob. Then we define information usage and show how it upper and lower bounds various measures of bias and estimation error due to data exploration in Section~\ref{sec:bounds}. In Section~\ref{selective}, we study specific examples of data exploration through the lens of information usage, which gives insight into Bob's practices of filtering, visualization, and maximum selection. Information usage naturally motivates randomization approaches to reduce bias and we explore this in Section~\ref{sec:randomization}. In Section ~\ref{sec:randomization}, we also study a model of a data analyst who--like Bob--interacts adaptively with the data many times before selecting values to report. 

\section{A Model of Data Exploration} \label{sec:model}
We consider a general framework in which a dataset $D$ is drawn from a probability distribution $\mathcal{P}$ over a set of possible datasets $\mathcal{D}$. The analyst is considering a large number $m$ of possible analyses on the data, but wants to report only the most interesting results. She  decides to report the result of a single analysis, and chooses which one \emph{after} observing the realized dataset, $D$, or some summary statistics of $D$. More formally, the data analyst considers $m$ functions $\phi_1,...,\phi_m : \mathcal{D} \rightarrow \mathbb{R}$ of the data, where $\phi_i(D)$ denotes the output of the $i$th analysis on the realization $D$.  Each function $\phi_i$ is typically called an estimator; each $\phi_i(D)$ is an estimate or statistic calculated from the  sampled data, and is a random variable due to the randomness in the realization of $D$. After observing the sampled-data, the analyst chooses to report the value $\phi_{T(D)}(D)$ for $T(D) \in \{1,...,m\}$. The selection rule $T: \mathcal{D} \rightarrow \{1,...,m\}$ captures how the analyst uses the data and chooses which result to report. Because the choice made by $T$ is itself a function of the sampled-data, the reported value $\phi_{T(D)}(D)$ may be significantly biased. For example, $\E[\phi_{T(D)}(D)]$  could be very far from zero even if each fixed function $\phi_{i}(D)$ has zero mean.

Note that although the number of estimators is assumed to be finite, it could be arbitrarily large; in particular $m$ can be exponential in the number of samples in the dataset. The $\phi_i$'s represent the set of all estimators that the analyst \emph{potentially} could have considered during the course of exploration. Also, while for simplicity we focus on the case where exactly one estimate is selected and reported, our results apply in settings where the analyst selects and reports many estimates.\footnote{For example, if the analyst chooses to report $m_0 \leq m$ results,  our framework can be used to bound the average bias of the reported values by letting $T$ be a random draw from the $m_0$ selected analyses.}

{\bf Example 1.}  \emph{For Alice, $D$ is a 1000-by-2003 matrix, where the rows are the individuals and the columns are the 2000 genes plus the three possible weight changes. Here there are $m = 2000$ potential estimators and $\phi_i$ is the correlation between the $i$th gene and the weight change between times 1 and 3. Alice's analysis corresponds to the selection procedure $T = \arg\max_i \phi_i$.}

{\bf Example 2.} \emph{Bob has the same dataset $D$. Because his exploration could have led him to use any of the three possible weight-change measures, the set of potential estimators are the correlations between the expression of one gene and one of the three weight changes and there are $2000 \times 3$ such $\phi_i$'s. Bob's adaptive exploration also corresponds to a selection procedure $T$ that takes the dataset and picks out a particular correlation value $\phi_T$ to report.}

{\bf Selection Bias.} Denote the \emph{true} value of estimator $\phi_i$ as $\mu_i \equiv \E[\phi_i(D)]$; this is the value that we expect if we apply $\phi_i$ on multiple independent replication datasets. On a particular dataset $D$, if $T(D)=i$ is the selected test, the output of data exploration is the value $\phi_i(D)$. The output and true-value can be written more concisely as $\phi_T$ and $\mu_T$. The difference $\phi_T - \mu_T$ captures the error in the reported value. We are interested in quantifying the \emph{bias} due to data-exploration, which is defined as the average error $\E[\phi_T - \mu_T]$. We will quantify other metrics of error, such as the expected absolute-error
$\E[|\phi_T - \mu_T|]$ or the squared-error $\E[(\phi_T - \mu_T)^2]$. In each case, the expectation is over all the randomness in the dataset $D$ and any intrinsic randomness in $T$.

\section{Related Work}
There is a large body of work on methods for providing meaningful statistical inference and preventing false discovery. Much of this literature has focused on controlling the false discovery rate in multiple-hypothesis testing where the hypotheses are not adaptively chosen \cite{benjamini1995controlling, benjamini2001control}. Another line of work studies confidence intervals and significance tests for parameter estimates in sparse high dimensional linear regression (see \cite{belloni2014inference, van2014asymptotically, javanmard2014confidence, lockhart2014significance} and the references therein). 

One recent line of work \cite{fithian2014optimal, Taylor2015} proposes a framework for assigning significance and confidence intervals in selective inference, where model selection and significance testing are performed on the same dataset.  These papers correct for selection bias by explicitly conditioning on the event that a particular model was chosen. While some powerful results can be derived in the selective inference framework (e.g. \cite{taylor2014exact, lee2016exact}),
it requires that the conditional distribution
$
\Prob(\phi_i = \cdot | T=i)
$
is known and can be directly analyzed. This requires that the candidate models and the selection procedure $T$ are mathematically tractable and specified by the analyst before looking at the data.
Our approach does not explicitly adjust for selection bias, but it enables us to formalize insights that apply to very general selection procedures. For example, the selection rule $T$ could represent the choice made by a data-analyst, like Bob,  after performing several rounds of exploratory analysis.


A powerful line of work in computer science and learning theory \cite{bousquet2002stability, poggio2004general, shalev2010learnability} has explored the role of algorithmic stability in preventing overfitting. Related to stability is PAC-Bayes analysis, which provides powerful generalization bounds in terms of KL-divergence \cite{mcallester2013pac}. There are two key differences between stability and our framework of information usage. First, stability is typically defined in the worst case setting and is agnostic of the data distribution. An algorithm is stable if, no matter the data distribution, changing one training point does not affect the predictions too much. Information usage gives more fine-grained bias bounds that depend on the data distribution. For example, in Section \ref{subsec: rank selection} we show the same learning algorithm has lower bias and lower information usage as the signal in the data increases. The second difference is that stability analysis has been traditionally applied to prediction problems---i.e. to bounding generalization loss in prediction tasks. Information usage applies to prediction---e.g. $\phi_i$ could be the squared loss of a classifier---but it also applies to model estimation where $\phi_i$ could be the value of the $i$th parameter. 

Exciting recent work in computer science \cite{blum2015ladder, hardt2014preventing, dwork2015generalization, dwork2015reusable} has leveraged the connection between algorithmic stability and differential privacy to design specific differentially private mechanisms that reduce bias in adaptive data analysis. 
In this framework, the data analyst interacts with a dataset indirectly, and sees only the noisy output of a differentially private mechanism. In Section \ref{sec:randomization}, we discuss how information usage also motivates using various forms of randomization to reduce bias. In the Appendix, we discuss the connections between mutual information and a recently introduced measure called max-information \cite{dwork2015reusable}.
The results from this privacy literature are designed for worst-case, adversarial data analysts. We provide guarantees that vary with the selection rule, but apply to all possible selection procedures, including ones that are not differentially private. 
The results in algorithmic stability and differential privacy are complementary to our framework: these approaches are specific techniques that guarantee low bias for worst-case analysts, while our framework quantifies the bias of any general data-analyst. 

Finally it is also important to note the various practical approaches used in specific settings to quantify or reduce bias from exploration. Using random subsets of data for validation is a common prescription against overfitting. This is feasible if the data points are independent and identically distributed samples. However, for structured data---e.g. time-series or network data---it is not clear how to create a validation set. The bounds on overfitting we derive based on information usage do not assume independence and apply to structured data. Special cases of selection procedures $T$ corresponding to filtering by summary statistics of biomarkers \cite{bourgon2010independent} and selection matrix factorization based on a stability criterion \cite{wu2016stability} have been studied. The insights from these specific settings agree with our general result that low information usage limits selection bias.

\section{Controlling Exploration Bias via Information Usage} \label{sec:bounds}

\textbf{Information usage upper bounds bias.} In this paper, we bound the degree of bias in terms of an information--theoretic quantity: the mutual information between the choice $T(D)$ of which estimate to report, and the actual realized value of the estimates $(\phi_{1}(D),...,\phi_{m}(D))$. We state this result in a general framework, where $\phibf=(\phi_1,...,\phi_m): \Omega \rightarrow \mathbb{R}^m$ and $T : \Omega \rightarrow \{1,..,m\}$ are any random variables defined on a common probability space. Let $\mubf=(\mu_1,...,\mu_m) \triangleq \E[ \phibf]$ denote the mean of $\phibf$. Recall that a real-valued random variable $X$ is $\sigma$--sub-Gaussian if for all $\lambda \in \mathbb{R}$,
$\E[e^{\lambda X}] \leq e^{\lambda^2 \sigma^2 /2} $  so that the moment generating function of $X$ is dominated by that of a normal random variable. Zero--mean Gaussian random variables are sub-Gaussian, as are bounded random variables.

\begin{prop}\label{prop: main result}
	If $\phi_i - \mu_i$ is $\sigma$--sub-Gaussian for each $i \in \{1,...,m\}$, then,
	\[
	|\E \left[\phi_T - \mu_T \right]| \leq \sigma\sqrt{2 I(T; \phibf)},
	\]
	where $I$ denotes mutual information\footnote{The mutual information between two random variables $X, Y$ is defined as $I(X;Y) = \sum_{x, y} \Prob(x,y) \log \left(\frac{\Prob(x,y)}{\Prob(x)\Prob(y)} \right)$.}.
\end{prop}


The randomness of $\phibf$ is due to the randomness in the realization of the data $D \sim \mathcal{P}$. This captures how each estimate $\phi_i$ varies if a replication dataset is collected, and hence captures the \emph{noise} in the statistics. The mutual information $I(T; \phibf)$, which we call \textbf{information usage}, then quantifies the dependence of the selection process on the noise in the estimates. Intuitively, a selection process that is more sensitive to the noise (high $I$) is at a greater risk for bias. We will also refer to $I(T; \phibf)$ as  bad information usage to highlight the intuition that it really captures how much information about the noise in the data goes into selecting which estimate to report. We normally think of data analysis as trying to extract the \emph{good} information, i.e. the true signal, from data. The more bad information is used, the more likely the analysis procedure is to overfit.  

When $T$ is determined entirely from the values $\{\phi_1, ..., \phi_m\}$,  mutual information $I(T; \phibf)$ is equal to entropy $H(T)$. This quantifies how much $T$ varies over different independent replications of the data. 

The parameter $\sigma$ provides the natural scaling for the values of $\phi_i$. The condition that $\phi_i$ is $\sigma$-sub-Gaussian ensures that its tail is not too 
heavy\footnote{A random variable $X$ is said to be $\sigma$-sub-Gaussian if $\E\left[  e^{\lambda(X - \E[X])}\right] \leq e^{\sigma^2\lambda^2/2}$ for all $\lambda$.}. In the Appendix, we show how this condition can be relaxed to treat cases where $\phi_i$ is a sub-Exponential random variables (Proposition~\ref{prop: subexponential}) as well as settings where the $\phi_i$'s have different scaling $\sigma_i$'s (Proposition~\ref{prop: main result 2}). 

Proposition~\ref{prop: main result} applies in a very general setting. The magnitude of overfitting depends on the generating distribution of the data set, and on the size of the data, and this is all implicitly captured in by the mutual-information $I(T; \phibf)$.  For example, a common type of estimate of interest is
$
\phi_i = n^{-1} \sum_{j=1}^{n} f_{i}(X_j),
$
the sample average of some function $f_i$ based on an iid sequence $X_1,...,X_n$. Note that if $f_i(X_j) - \E[f_i(X_j)]$ is sub-Gaussian with parameter $\sigma$, then $\phi_i - \mu_i$ is sub-Gaussian with parameter $\sigma/\sqrt{n}$ and therefore
\[
|\E[\phi_{T}] - \E[\mu_T] | \leq \sigma \sqrt{\frac{2I(T; \phibf)}{n}}.
\]

To illustrate Proposition~\ref{prop: main result}, we consider two extreme settings: one where $T$ is chosen independently of the data and one where $T$ heavily depends on the values of all the $\phi_i$'s. The subsequent sections will investigate the applications of information usage in depth in settings that interpolate between these two extremes. 

\textbf{Example: data-agnostic exploration.} 
Suppose $T$ is independent of $\phibf$. This may happen if the choice of which estimate to report is decided ahead of time and cannot change based on the actual data. It may also occur when the dataset can be split into two statistically independent parts, and separate parts are reserved for data-exploration and estimation.  In such cases, one expects there is no bias because the selection does not depend on the actual values of the estimates. This is reflected in our bound: since $T$ is independent of $\phibf$, $I(T; \phibf) =0$ and therefore $\E [\phi_{T}]= \E[\mu_{T}]$. 

\textbf{Example: maximum of Gaussians.}
Suppose each $\phi_i$ is an independent sample from the zero-mean normal $\mathcal{N}(0,\sigma^2)$. If $T= \underset{1\leq i \leq m}{\arg\max} \phi_{i}$, then $I(T; \phibf)= H(T) = \log(m)$ because all $m$ $\phi_i$'s are symmetric and have equal chance of being selected by $T$. Applying Proposition~\ref{prop: main result} gives
$
\E [\phi_T - \mu_T] = \E[\phi_T] \leq \sigma\sqrt{2 \log(m)}.
$
This is the well known inequality for the maximum of Gaussian random variables. Moreover, it is also known that this equation approaches equality as the number of Gaussians, $m$, increases, implying that the information usage $I(T; \phibf)$ precisely measures the bias of max-selection in this setting. 
It is illustrative to also consider a more general selection $T$ which first ranks the $\phi_i$'s from the largest to the smallest and then uniformly randomly selects one of the $m_0$ largest $\phi_i$'s to report. Here $I(T; \phibf) = H(T) - H(T|\phibf) $, where $H(T) = \log m$ (by the symmetry of $\phi_i$ as before) and $H(T | \phibf) = \log m_0$ (since given the values of $\phi_i$'s there is still uniform randomness over which of the top $m_0$ is selected). We immediately have the following corollary. 
\begin{cor}\label{cor: orderstat}
	Suppose for each $i\in \{1, ..., m\}$, $\phi_i$ is a zero-centered sub-Gaussian random variable with parameter $\sigma$. Let $\phi_{(1)} \geq \phi_{(2)} \geq ... \geq \phi_{(m)}$ denote the values of $\phi_i$ sorted from the largest to the smallest. Then
	\[
	\E \left[ \frac{1}{m_0} \sum_{i = 1}^{m_0} \phi_{(i)} \right] \leq \sigma \sqrt{2 \log \frac{m}{m_0}}.
	\]
\end{cor}
In Appendix \ref{sec_app:lowerbound}, we show that this bound is also tight as $m$ and $m_0$ increase.

\textbf{Information usage bounds other metrics of exploration error.} So far we have discussed how mutual information upper bounds the bias $| \E\left[ \phi_T - \mu_T  \right]|$. In different application settings, it might be useful to control other measures of exploration error, such as the absolute error deviation $\E\left[ | \phi_T - \mu_T | \right]$ and the squared error $\E\left[ (\phi_T - \mu_T)^2 \right] $. 

Here we extend Proposition~\ref{prop: main result} and show how $\sqrt{I(T; \phibf)}$ and $I(T; \phibf)$ can be used to bound  absolute error deviation and squared error. Note that due to inherent noise even in the absence of selection bias,  the absolute or squared error can be of order $\sigma$ or $\sigma^2$,  respectively. The next result effectively bounds the additional error introduced by data-exploration in terms of information-usage.  
\begin{prop}\label{prop: absolute and squared error}
	Suppose for each $i \in \{1,...,m\}$, $\phi_i - \mu_i$ is $\sigma$ sub-Gaussian. Then
	\[
	\E[| \phi_T - \mu_T |] \leq \sigma + c_1 \sigma \sqrt{2 I(T; \phibf)}
	\]
	and
	\[
	\E[ (\phi_T - \mu_T)^2]  \leq 1.25\sigma^2 + c_2\sigma^2 I(T; \phibf) .
	\]
	where $c_1<36$ and $c_2\leq 10$ are universal constants. 
\end{prop}


\textbf{Information usage also lower bounds error.}
In the maximum of Gaussians example, we have already seen a setting where information usage precisely quantifies bias. Here we show that this is a more general phenomenon by exhibiting a much broader setting in which mutual-information lower bounds expected-error. This complements the upper bounds of Proposition~\ref{prop: main result} and Proposition~\ref{prop: absolute and squared error}. 

Suppose $T=\arg\max_{i} \phi_i$ where $\phibf \sim \mathcal{N}(\mubf, I)$.  Because $T$ is a deterministic function of $\phibf$, mutual information is equal to entropy.  The probability $T=i$ is a complicated function of the mean vector $\mubf$, and the entropy $H(T)$ provides a single number measuring the uncertainty in the selection process. Proposition \ref{prop: absolute and squared error} upper bounds the average squared distance between $\phi_T$ and $\mu_T$ by entropy. The next proposition provides a matching lower bound, and therefore establishes a fundamental link between information usage and selection-risk in a natural family of models. 
\begin{prop}\label{prop: lower bound} Let $T=\arg\max_{1\leq i \leq m} \phi_i$ where $\phibf \sim \mathcal{N}(\mubf, I)$. There exist universal numerical constants $c_1=1/8$,  $c_2 < 2.5$ , $c_3=10$, and $c_4=1.5$ such that for any $m\in \mathbb{N}$ and $\mubf \in \mathbb{R}^{m}$,   
	\[
	c_1 H(T)- c_2 \leq \E[(\phi_T - \mu_T)^2]  \leq c_3H(T) + c_4 . 
	\] 
\end{prop}
Recall that the entropy of $T$ is defined as \[
H(T) = \sum_{i}\Prob(T=i) \log\left(\frac{1}{\Prob(T=i)}\right).
\]
Here $\log(1/ \Prob(T=i))$ is often interpreted as the ``surprise'' associated with the event $\{T=i\}$ and entropy is interpreted as expected surprise in the realization of $T$. Proposition~\ref{prop: lower bound} relies on a link between the surprise associated with the selection of statistic $i$, and the squared error $(\phi_i-\mu_i)^2$ on events when it is selected.  

To understand this result, it is instructive to instead consider a simpler setting; imagine $m=2$ , $\phi_1=x$ always, $\phi_2 \sim \mathcal{N}(0, 1)$, and the selection rule is $T=\arg\max_{i} \phi_i$. When $x>>0$ is large, 
\[ 
\log(1/\Prob(T=2) ) = \log(1/ \Prob(\phi_2 \geq x)) \approx x^{2}/2
\]
and so the surprise associated with the event $\{T=2\}$ scales with the squared gap between the selection threshold $x$ and the true mean of $\phi_2$. One can show that as $x\rightarrow \infty$,
\begin{eqnarray*}
	H(T_x) &\sim& \Prob(T_x=2)\log(1/\Prob(T_x=2)) \\
	&\sim&  \Prob(T_x=2)x^{2} \\
	&\sim& \E[(\phi_{T_x} - \mu_{T_x})^2] 
\end{eqnarray*}
where $T_x$ denotes the selection rule with threshold $x$ and $f(x) \sim g(x)$ if $f(x)/g(x) \to 1$ as $x\to \infty$.  

In the Appendix, we investigate additional threshold-based selection policies applied to Gaussian and exponential random variables, allowing for arbitrary correlation among the $\phi_i$'s, and show that $H(T)$ also provides a natural lower bound on estimation-error. 


\section{When is bias large or small? The view from information usage} \label{selective}
In this section, we consider several simple but commonly used procedures of feature selection and parameter estimation. In many applications, such feature selection and estimation are performed on the same dataset. Information usage provides a unified framework to understand selection bias in these settings. Our results inform when these these procedures introduce significant selection bias and when they do not. The key idea is to understand which structures in the data and the selection procedure make the mutual information $I(T ; \phibf)$  significantly smaller than the worst-case value of $\log(m)$. We provide several simulation experiments as illustrations.

\subsection{Filtering by marginal statistics}\label{subsec: filtering}
Imagine that $T$ is chosen after observing some dataset $D$. This dataset determines the values of  $\phi_1,...,\phi_m$, but may also contain a great deal of other information. Manipulating the mutual information shows
\begin{eqnarray*}
	I(T; \phibf) &=& H(T) - H(T| \phibf) \\
	&\leq& H(T) - I(T;  D |  \phibf) \\
	&=& (1-\alpha)H(T)
\end{eqnarray*}

where $\alpha = I(T;  D |  \phibf)/ H(T)$ captures the fraction of the uncertainty in $T$ that is explained by the data in $D$ beyond the values $\phi_1,...,\phi_m$.  In many cases, instead of being a function of $\phibf$, the choice $T$ is a function of data that is more loosely coupled with $\phibf$, and therefore we expect that $I(T; \phibf)$ is much smaller than $H(T)$ (which itself can be less than $\log(m)$).

One setting when the selection of $T$ depends on the statistics of $D$ that are only loosely coupled with $\phibf$ is variance based feature selection \cite{ewasher, variancefilter}. Suppose we have $n$ samples and $m$ bio-markers. Let $X_{i,j}$ denote the value of the $i$-th bio-marker on sample $j$. Here $D = \{X_{i,j}\}$. Let $\phi_i = n^{-1} \sum_{j=1}^{n} X_{i,j}$ be the empirical mean values of the $i$-th biomarker. We are interested in identifying the markers that show significant non-zero mean. Many studies first perform a filtering step to select only the markers that have high variance and remove the rest. The rationale is that markers that do not vary could be measurement errors or are likely to be less important. A natural question is whether such variance filtering introduces bias.

In our framework, variance selection is exemplified by the selection rule $T = \arg\max_{i} V_i$ where $V_i =  \sum_{j=1}^{n} (X_{i,j}- \phi_i)^2$. Here we consider the case where only the marker with the largest variance is selected, but all the discussion applies to softer selection when we select the $K $ markers with the largest variance. 
The resulting bias is $\E[\phi_T - \mu_T]$. Proposition \ref{prop: main result} states that variance selection has low bias if $I(T; \phibf)$ is small, which is the case if the empirical means and variances, $\phi_i$ and $V_i$, are not too dependent. In fact, when the $X_{i,j}$ are i.i.d. Gaussian samples, $\phi_1,...,\phi_m$ are independent of $V_1,...,V_m$ . Therefore $I(T; \phibf)=0$ and we can guarantee that there is no bias from variance selection.

This illustrates an important point that the bias bound depends on $I(T; \phibf)$ instead of $I(T; D)$. The selection process $T$ may depend heavily on the dataset $D$ and $I(T;D)$ could be large. However as long as the statistics of the data used for selection have low mutual information with the estimators $\phi_i$, there is low bias on the reported values. 

We can apply our framework to analyze biases that arise from feature filtering more generally. A common practice in data analysis is to reduce multiple hypotheses testing burden and increase discovery power by first filtering out covariates or features that are unlikely to be relevant or interesting  \cite{bourgon2010independent}. This can be viewed as a two-step procedure. For each feature $i$, two marginal statistics are computed from the data, $\psi_i$ and $\phi_i$. Filtering corresponds to a selection protocol on $\psi_i$. Since $I(T; \phibf) \leq I(\bm \psi; \phibf)$, if the $\psi_i$'s do not reveal too much information about $\phi_i$'s then the filtering step does not create too much bias.
In our example above, $\psi_i$ is the sample variance and $\phi_i$ is the sample mean of feature $i$. General principles for creating independent $\psi_i$ and $\phi_i$ are given in \cite{bourgon2010independent}.

More generally, suppose the dataset determines two  sets of statistics $\phibf=(\phi_1,..,\phi_m)$ and ${\bm \psi}=(\psi_i,...,\psi_{m'})$. We report $\phi_T$ and want to quantify its bias, but the selection rule depends only on the $\psi_i$'s,  i.e. $T=f(\psi_i)$ can be expressed as a function of the $\psi_i$'s.  This captures the general situation where data processing and feature selection uses one set of summary statistics ($\psi$) and we want to quantify the bias introduced in these steps on another set of statistics ($\phi$). 
The dependence structure can be expressed as a Markov chain $T -{\bm \psi} -\phibf$, where this notation indicates that conditioned on ${\bm \psi}$, $T$ is independent of $\phibf$.
The data processing inequality implies $I(T; \phibf) \leq I(\phibf; \psi)$, which--combined with our bound--formalizes the intuition that the selection rule cannot be substantially biased when $\phibf$ and ${\bm \psi}$ share limited information in common. However, this bound may be quite loose. We instead turn to strong data processing inequalities. 
\begin{defn}
	A pair of random variables $(X,Y)$ satisfies a strong data-processing inequality with contraction coefficient $\eta \in [0,1]$ if for all random variables $U$ with $U-X-Y$ 
	\begin{equation*}\label{eq: STDP}
	I(U;Y) \leq \eta I(U;X)
	\end{equation*}
	Let $\eta_{XY}$ be the smallest constant such that \eqref{eq: STDP} is satisfied for all valid $U$.  
\end{defn}

The contraction coefficient satisfies several natural properties. First, it \emph{tensorizes} \cite{anantharam2013maximal}. That is, if $(X_1,Y_1),...(X_n, Y_n)$ is an independent sequence, then $\eta_{XY} = \max_{i} \eta_{X_{i}Y_{i}}$. Also, if $X,Y$ and $Z$ follow a Markov chain $X-Y-Z$ then $\eta_{XZ}\leq \eta_{YZ}$. 

\textbf{Example.} Suppose $D=(X_1,...,X_n)$ consists of $n$ iid random variables and ${\bm \psi}=(X_1,...,X_k)$ is a subsample of $k<n$ data points. Then $\eta_{{\bm \psi} \phibf} \leq \eta_{{\bm \psi} D} \leq  k/n$ \cite{kamath2015strong}.

\textbf{Example.}(Noisy Channels) If $(X,Y)$ corresponds to a binary symmetric channel with error rate $\delta$ then $\eta_{XY}=(1-2\delta)^2$ \cite{polyanskiy2016dissipation}.

Note that the contraction coefficient $\eta_{\phibf {\bm \psi}}$ depends only on the distribution of $\phibf$ and ${\bm \psi}$, and not on the selection rule $T$. A benefit of our mutual information framework for bounding the exploration bias is that we can immediately apply Strong Data Processing to obtain tighter bounds on bias:   
\begin{prop} \label{prop:strongdataproc}
	Suppose $\phi_i - \mu_i$ is $\sigma$ sub-Gaussian for each $i\in \{1,..,,m\}$. Then if the selection $T$ is independent of $\phibf$ conditioned on ${\bm \psi}$, 
	\[
	\E[\phi_T - \mu_T] \leq \sigma \sqrt{2 \eta_{{\bm \psi} \phibf} I(T; \mathbb{\psi}) }.
	\]
\end{prop}

\subsection{Bias due to data visualization}
Data visualization, using clustering for example, is a common technique to explore data and it can inform subsequent analysis. How much selection bias can be introduced by such visualization? While in principle a visualization could reveal details about every data point, a human analyst typically only extracts certain salient features from plots. For concreteness, we use clustering as an example, and imagine the analyst extracts the number of clusters $K$ from the analysis. In our framework the natural object of study is the information usage $I(K; \phibf)$, since if the final selection $T$ is a function of $K$, then $I(T; \phibf) \leq I(K; \phibf)$ by the data-processing inequality. In general, $K$ is a random variable that can take on values 1 to $n$ (if each point is assigned its own cluster). When there is structure in the data and the clustering algorithm captures it, then $K$ can be strongly concentrated around a specific number of clusters and $I(K; \phibf) \le H(K) \approx 0$. In this setting, clustering is informative to the analyst but does not lead to ``bad information-usage'' and therefore does not increase exploration bias.  This is a stylized example; if the analyst uses additional information beyond the number of clusters $K$, then the bias could increase.

\subsection{Rank selection with signal}\label{subsec: rank selection}
Rank selection is the procedure for selecting the $\phi_i$ with the largest value (or the top $K$ $\phi_i$'s with the largest values). It is the simplest selection policy and the one that we are instinctively most likely to use. We have seen previously how rank selection can introduce significant bias. In the bio-marker example in Subsection \ref{subsec: filtering}, suppose there is \emph{no signal} in the data, so $X_{i,j} \sim \mathcal{N}(0,1)$ and $\phi_i \sim \mathcal{N}(0,1/n)$. Under rank selection, $\phi_T$ would have a bias close to $\sqrt{(2\log m)/n}$.

What is the bias of rank selection when there \emph{is} signal in the data? Our framework cleanly illustrates how signal in the data can reduce rank selection bias. As before, this insight follows transparently from studying the mutual information $I(T, \phibf)$. Recall that mutual information is bounded by entropy:
$
I(T; \phibf) \leq H(T) \leq \log(m).
$
When the data provides a strong signal of which $T$ to select, the distribution of $T$ is far from uniform, and $H(T)$ is much smaller than its worst case value of $\log(m)$.


Consider the following simple example. Assume
\[
\phi_i \sim \begin{cases}
\N(\mu, \sigma^2) & \text{If } i=I^*\\
\N(0, \sigma^2) & \text{If } i\neq I^*
\end{cases}
\]
where $ \mu \geq 0$. The data analyst would like to identify $I^*$ and report the value of $\phi_{I^*}$. To do this, she selects $T=\arg\max_{i} \phi_i$. When $ \mu=0$, there is no true signal in the data and $T$ is equally likely to take on any value in $\{1,..,m\}$, $I(T; \phibf)=H(T) = \log(m)$. As $\mu$ increases, however, $T$ concentrates on $I^*$, causing $H(T)$ and the bias $\E [\phi_T - \mu_T]$ to diminish. We simulated this example with $m = 1000$ $\phi_i$'s, all but one of which are i.i.d. samples from $\N(0,1)$ and $\phi_{I^*} \sim \N(\mu, 1)$ for $\mu \in [1,4]$. The simulation results, averaged over 1000 independent runs, are shown in Figure~\ref{fig:signal}.

\begin{figure}[!htb]
	\centering
	\includegraphics[ clip,scale=0.4]{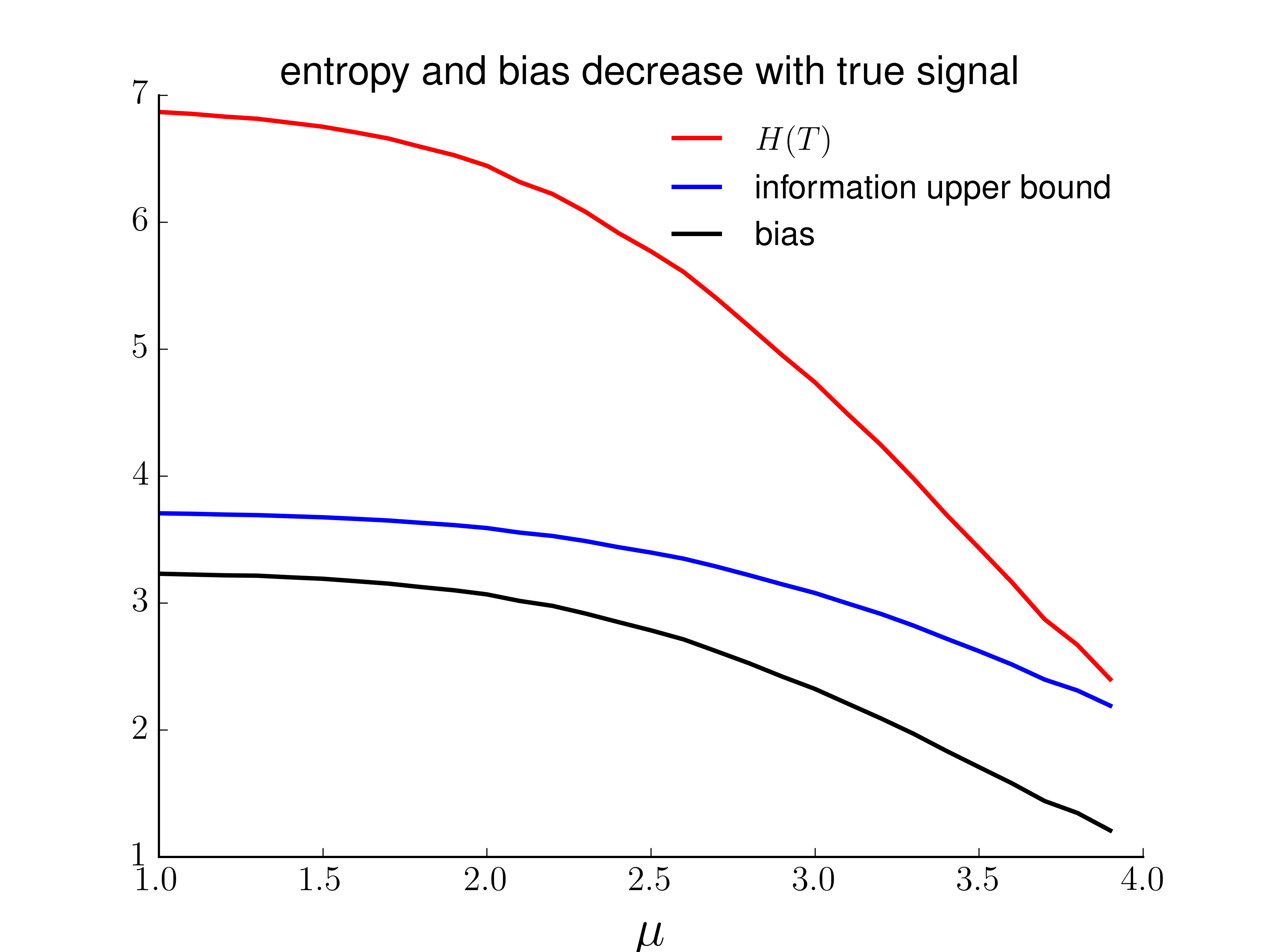}
	\caption{As the signal strength increases ($\mu$ increases), the entropy of selection $H(T)$ decreases, causing the information upper bound $\sqrt{2I(T; \phibf)}$ to also decrease. The bias of the selected $\phi_T$ decreases as well. }
	\label{fig:signal}
\end{figure}




%

\subsection{Information usage along the Least Angle Regression path} 
Our analyses illustrate that in certain stylized settings, information usage tightly bounds the bias of optimization selections. Here we show that information usage also accurately captures the bias of a more complex selection procedure corresponding to Least Angle Regressions (LARS) \cite{efron2004least}. LARS is an interesting example for two reasons. First it is widely used as a practical tool for sparse regression and is closely related to LASSO. Second LARS composes a sequence of maximum selections and thus provides a more complex example of selection. In Figure~\ref{fig:lars_example}, we show the simulation results for LARS under three data settings corresponding to low, medium and high signal-to-noise ratios. We use bootstrapping to empirically estimate the information usage and since we know the ground truth of the experiment, we can easily compute the bias of LARS. As the signal in the data increases, the information usage of LARS decreases and, consistent with the predictions of our theory, the bias of LARS also decreases.  Moreover, as the number of selected features increases, the average (per feature) information usage of LARS decreases and, consistent with this, the average bias of LARS also decreases monotonically. Details of the experiment are in the Appendix. 

\subsection{Differentially private algorithms} 
Recent papers \cite{dwork2014preserving, dwork2015reusable} have shown that techniques from differential privacy, which were initially inspired by the need to protect the security and privacy of datasets, can be used to develop adaptive data analysis algorithms with provable bounds on over-fitting. These differentially private algorithms satisfy worst case bounds on certain likelihood ratios, and are guaranteed to have low information-usage. On the other hand, many algorithms have low information-usage without being differentially private. Moreover, as we have seen, the exploration bias of an algorithm could be large or small depending on the particular dataset (e.g. the signal-to-noise ratio of the data) and information usage captures this. Differentially private algorithms have low information usage for \emph{all} datasets and $T$ that is designed adversarial to exploit this dataset, so this is a much stricter condition.  In \cite{dwork2015reusable}, the authors also define and study a notion of max-information, which can be viewed as a worst-case analogue of mutual information. We discuss the relationship between these measures further in the Appendix.

\begin{figure}[!htb]
	\centering
	\includegraphics[ clip,scale=0.4]{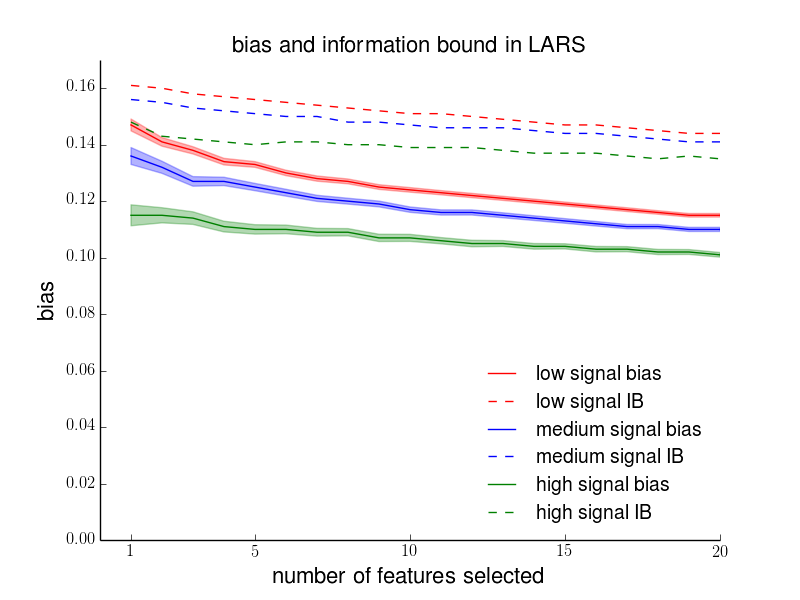}
	\caption{Information bound $\sqrt{2 I (T; \phibf)}$ (dotted lines) and bias of Least Angle Regression (solid lines). Results are shown for low (red), medium (blue) and high (green) signal-to-noise settings. The $x$-axis indicates  the number of features selected by LARS and the $y$-axis corresponds to the average information usage and bias in the selected features.}
	\label{fig:lars_example}
\end{figure}

\subsection{Information usage and classification overfitting}

This section applies our framework  to the problem of overfitting in classification. A classifier is trained on a dataset consisting of $n$ examples, with input features $X_1,..,X_n \in \mathcal{X}$ and corresponding labels $Y_1,...Y_n \in \{-1, 1\}$. We consider here a setting where the features of the training examples $X_i=x_i$ are fixed, and study overfitting of the noisy labels. Each label $Y_i$ is drawn independently of the other labels from an unknown distribution $\Prob(Y_i =1 | X_i=x_i)$. 
A classifier $f$ associates a label $f(x) \in \{-1,1\}$ with each input $x$. The training error of a fixed classifier $f$ is 
\[
\hat{L}(f) = \frac{1}{n} \sum_{i=1}^{n} \mathbf{1}(f(x_i) \neq Y_i)
\]
while its true error rate is
\[
L(f)=\E[\hat{L}(f)] = \frac{1}{n} \sum_{i=1}^{n} \Prob( f(x_i) \neq Y_i),
\]
is the expected fraction of examples it mis-classifies on a random draw of the labels $Y_1,..,Y_n$. The process of training a classifier corresponds to selecting, as a function of the observed data, a particular classification rule $\hat{f}$ from a large family $\mathcal{F}$ of possible rules. Such a procedure may overfit the training data, causing the average training error $\E[\hat{L}(\hat{f})]$ to be much smaller than its true error rate $\E[L(\hat{f})]$.   

As an example, suppose each $X_i \in \mathbb{R}^{d}$ is a $d$--dimensional feature vector, and $\mathcal{F}=\{f_{\theta}: \theta \in \mathbb{R}^{d}\} $ consists of all linear classifiers of the form $f_{\theta}(x) = \mathbf{1}(x^T \theta \geq 0)$. A training algorithm might set $\hat{f}=f_{\hat{\theta}}$ by choosing the parameter vector that minimizes the number of mis-classifications on the training set. This procedure tends to overfit the noise in the training data, and as a result the average training of $\hat{f}$ can be much smaller than its true error rate.  The risk of  over-fitting tends to increase with the dimension $d$, since higher dimensional models allow the algorithm to fit more complicated, but spurious, patterns in the training set. 

The field of statistical learning provides numerous bounds on the magnitude of overfitting based on more general notions of the complexity of an arbitrary function class $\mathcal{F}$, with the most influential being the Vapnik-Chervonenkis dimension, or VC-dimension\footnote{The VC-dimension of $\mathcal{F}$ is the size of the largest set it shatters. A set $\{x_1,..,x_m\}\in \mathcal{X}$ is shattered by $\mathcal{F}$ if for any choice of labels $y_1,..,y_m \in \mathcal{Y}$, there is some $f\in \mathcal{F}$ with $f(x_i)=y_i$ for all $i$.}.  While the focus is on overfitting of the training data, similar concerns apply to overfitting the validation data. 

The next proposition provides information-usage bounds the degree of over-fitting, and then shows that mutual information is upper-bounded by the VC--dimension of $\mathcal{F}$. Therefore, information-usage is always constrained by function-class complexity. 
\begin{prop}\label{prop: overfitting}
	Let $\mathbf{x}\equiv(x_1,...,x_n)$, $\mathbf{Y}\equiv(Y_1,...,Y_n)$,  $\hat{f}(\mathbf{x}) \equiv (\hat{f}(x_1),...\hat{f}(x_n))$ and $\log_{+}(z) \equiv \max\{1, \log(z)\}$. Then,
	\[
	\E[L(\hat{f}) - \hat{L}(\hat{f})] \leq \sqrt{\frac{I(\hat{f}(\mathbf{x}); \mathbf{Y})}{2n}}.
	\]
	If $\mathcal{F}$ has VC-dimension $d<\infty$, then
	\[ 
	I(\hat{f}(\mathbf{x}); \mathbf{Y}) \leq d\log_{+}\left(\frac{ne}{d} \right).
	\]
\end{prop}
The proof of the information usage bound follows by an easy reduction to Proposition \ref{prop: main result}. The proof of the second claim relies on a known link between VC-dimension and a notion of the log-covering numbers of the function-class. 

It is worth highlighting that because VC-dimension depends only on the class of functions $\mathcal{F}$, bounds based on this measure can't shed light on which types of data-generating distributions and fitting procedures $ (\mathbf{X}, \mathbf{Y}) \mapsto \hat{f}$ allow for effective generalization. Information usage depends on both, and a result could be much smaller than VC-dimension; for example, this occurs when some classifiers in $\mathcal{F}$ are much more likely to be selected after training than others. This can occur naturally due to properties of the training procedure, like regularization, or properties of the data-generating distribution.

\subsection{Approximately independent data splitting.}
A data scientist has access to data in the form of $n$ samples $(s_1,\ldots,s_n)$ from a Markov chain. She would like to mimic the honest data-splitting she uses with i.i.d data. To do this, she splits the into three parts: $(s_1,..,s_{n_1})$, $(s_{n_1+1},..,s_{n_2})$ and $(s_{n_2+1},..,s_{n}).$  The first part is used for selection, the third for estimation, and the middle data is thrown away. In particular,  $\phi= \phi_1,...,\phi_{m} : (s_{n_2+1},..,s_{n})  \to \mathbb{R}^m$ and that $T: (s_1,..,s_{n_1}) \mapsto \{1,\ldots, m\}$.  One expects that if $n_2 - n_1$ is large so there is a sufficient delay between the two samples, then the risk of bias and overfitting will be low. We'll see that this is easy to formalize via an information usage lens. 

We assume the Markov process is stationary and time homogeneous with stationary distribution $\pi$. Moreover, it satisfies a uniform mixing condition 
\[
\max_{s} D( \Prob(s_{\tau}=\cdot | s_1=s) \, || \, \pi ) \leq c_0 e^{-c_1 \tau} \qquad \forall \tau \in \mathbb{N}. 
\]
We then claim that 
\[
I( T ; \phibf) \leq c_0 e^{-c_1 (n_2 - n_1) }
\]
and so a sufficient delay between the sample used for selection and the sample used for estimation guarantees low bias. 
We have immediately that, 
\begin{align*}
I( s_{t+\tau} ; s_t) &= \sum_{s} \Prob(s_t = s)  D( \Prob(s_{t+\tau}=\cdot | s_t=s) \, || \, \Prob(s_t=\cdot) ) \\
& \leq  c_0 e^{-c_1 \tau}
\end{align*}
where we used that $\Prob(s_t=s)=\pi(s)$. Then, by the data processing inequality 
\begin{align*} 
I( T ; \phibf) &\leq I( (s_1,..,s_{n_1}) \,;\,  (s_{n_2+1},..,s_{n})  ) \\
&\leq I(s_{n_1} ; s_{n_2+1}) \\
&\leq  c_0 e^{-c_1 (n_2-n_1)}. 
\end{align*}

\subsection{Bias control via FDR control}
There has been intense interest in large-scale hypothesis testing procedures that control the false-discovery rate. Here we consider the bias and error incurred when estimation is performed after variables are selected in this manner, and bound this in terms of the false discovery rate and the rates of type I and type II errors. 

As motivation, consider analysis of a large micro-array experiment. There is a large set of gene-expression data $D \in \mathbb{R}^{n \times m}$ consisting of $m$ gene expression levels drawn from $n$ samples, where there first $n_1$ samples were taken from tissue with a cancerous tumor and the remaining $n_2=n-n_1$ were taken from healthy tissue.  A scientist would like to identify genes with large differential between the expression levels across the two tissue types. She casts this as a multiple hypothesis testing problem, where rejecting a given null hypothesis indicates strength of evidence that an observed differential is unlikely due to random chance. Many procedures exist to control the false discovery rate, which is the expected proportion of type I errors among rejected null hypotheses. 

Consider for example the procedure proposed by Benjamini and Hochberg. One first constructs p-values $p_{1},\ldots p_{m}$ for $m$ separate hypothesis testing problems. These are then sorted as $p_{(1)} \leq p_{(2)} \leq \ldots, p_{(m)}$. To guarantee the false discovery rate is controlled at some level $q\in (0,1)$, their procedure specifies the selection of the the first $\hat{t}$ hypotheses, where $\hat{t}$ is the largest number such that $p_{(t)} \leq q t/m$. Framed differently, all hypotheses with p-values less than a random threshold $\hat{l}= q\hat{t}/m$ are rejected. To gain some insight, let us consider a simple model where each $p$-value is drawn either from a uniform distribution (i.e. the null distribution) or an alternative distribution $F$. Consider an asymptotic regime where the number of alternative $m\to \infty$, but the proportion of alternatives following the null distribution stays fixed. Then \cite{genovese2002operating} show that under regularity conditions on $F$, the random threshold $\hat{l}$ converges in probability to a deterministic limit $l^*$. Therefore, the rate of type I and type II errors, as well as the proportion of false discoveries, all tend to a fixed levels asymptotically as $m \to \infty$. Whether a particular hypothesis is accepted or rejected is still random and data-dependent, but when $m$ is large the overall proportions are nearly deterministic. 

We consider a more abstract framework. There is some random matrix $D\in \mathbb{R}^{n\times m}$, and a vector $\phi \in \mathbb{R}^m$ that is a function of $D$ with $\mubf = \E[\phibf]$. The indices $\{1,\ldots, m\}$ are partitioned into two sets $\mathcal{H}_0$ and $\mathcal{H}_1$. A selection procedure is a map $\psi: \mathbb{R}^{n \times m} \to \{0,1\}^m$, where $\psi(D)_{i} =1$ indicates variable $i$ was selected. We set $S_1\subset \{1,\ldots, m \}$ to be the set of selected variables and $S_0$ to be its complement. 

To form the analogy with the story above, we think of $\phi$ as a vector of summary statistics of the columns of $D$---e.g. the observed gene expression differential between tumor tissue and healthy tissue---and think of $\mathcal{H}_0$ as the set for which the null distribution holds --- e.g. across repeated samples there would not be an observed differential. The selected variables $S_1$ is the set for which the null hypothesis was rejected. 
Set $\hat{\alpha} = \#(\Hc_0 \cap S_1) / \#\Hc_0$ and $\hat{\beta} = \#(\Hc_1 \cap S_0) / \#\Hc_1$ to be analogues of the proportion of type I and type II errors. Note that $\hat{\alpha}$ is the fraction of false discoveries relative to the total number of nulls, and is different from what is called the False Discovery Proportion or FDP. 
To simplify the discussion, we assume there is always at least one selected variable, so $S_1$ is nonempty.  We are interested in the average error or bias in reported estimates among selected, which leads to the study of quantities like 
\begin{align}\label{eq: bias like terms for FDR}
&\frac{1}{\#S_1} \sum_{i\in S_1} (\phi_i - \mu_i), \\\nonumber
& \frac{1}{\#S_1} \sum_{i\in S_1} |\phi_i - \mu_i|, \\\nonumber 
\text{or} \quad &  \frac{1}{\#S_1} \sum_{i\in S_1} (\phi_i - \mu_i)^2.
\end{align}
These can be rewritten as $\E[\phi_{T} - \mu_T]$, $\E[|\phi_{T} - \mu_T|]$ or $\E[(\phi_{T} - \mu_T)^2]$ where, conditioned on $D$, $T$ is drawn uniformly at random from the set of selected of selected variables $S_1$. This leads naturally to the study of information usage $I(T ; \phibf)$, which bounds these quantities. The quantities in \eqref{eq: bias like terms for FDR} reflect whether, the estimation procedures applied to the selected variables produce accurate results \emph{on average}. For this reason, we are able to provide meaningful guarantees that do not degrade as $m\to \infty$, a regime in which it is impossible to guarantees that \emph{every} selected variable is estimated accurately.

Now, let us define ${\rm FDR} = \Prob(T \in \Hc_0)$ to be the false discovery rate. This is the expected proportion of selected variables $S_1$ that are contained within the null set $\Hc_0$.  The next lemma bounds information usage in terms of the false discovery rate, the rates of type I and II error, and an extra error term that vanishes as the random proportion of realized type I and II errors concentrate around their expected value. A short proof is given in Appendix \ref{sec_app:FDR}.
\begin{prop}\label{prop: FDR} For the FDR control problem defined above,
	\begin{align*}
	I(T; \phibf) \leq h({\rm FDR}) &+\left(1- {\rm FDR}\right)\cdot \log\left(\frac{1}{1- \beta}  \right) \\
	&+ {\rm FDR} \cdot  \log\left( \frac{1}{\alpha} \right) + \xi
	\end{align*}
	where $h(p)=-p\log(p)-(1-p)\log(1-p)$ denotes the binary entropy function, $\alpha=\E[\hat{\alpha}]$ and $\beta=\E[\hat{\beta}]$ denote the type I and II error proportion relative to the total number of true null  and true alternative, respectively. The error term is 
	\[ 
	\xi = \E\left[\log_{+}\left( \frac{1-\beta}{1-\hat{\beta}}  \right) \right]+\E\left[\log_{+}\left( \frac{\alpha}{\hat{\alpha}}  \right) \right]. 
	\]
	for $\log_{+}(x)\equiv \max\{0, \log(x)\}$. 
\end{prop}
This result further formalizes the insight that estimation after selection is unlikely to overfit in settings where the selection procedure works reliably. When the rates of false discovery, type I error, and type II error are small, information usage is guaranteed to also be low. The implied bounds on estimation error after selection grow smoothly as the reliability of the selection procedure degrades.

\section{Limiting information usage and bias via randomization}\label{sec:randomization}
We have seen how information usage provides a unified framework to investigate the magnitude of exploration bias across different analysis procedures and datasets. It also suggests that methods that reduces the mutual information between $T$ and $\phibf$ can reduce bias. In this section, we explore simple procedures that leverages randomization to reduce information usage and hence bias, while still preserving the utility of the data analysis. 

We first revisit the rank-selection policy considered in the previous subsection, and derive a variant of this scheme that uses randomization to limit information usage.
We then consider a model of a human data analyst who interacts sequentially with the data. We use a stylized model to show that, even if the analysts procedure is unknown or difficult to describe, adding noise during the data-exploration process can provably limit the bias incurred. 
Many authors have investigated adding noise as a technique to reduce selection bias in specialized settings \cite{dwork2014preserving, chaudhuri2011differentially}. The main goal of this section is to illustrate how the effects of adding noise is transparent through the lens of information usage.



\subsection{Regularization via randomized selection}
Subsection \ref{subsec: rank selection} illustrates how signal in the data intrinsically reduces the bias of rank selection by reducing the entropy term $H(T)$ in $I(T; \phibf) = H(T) - H(T | \phibf)$. A complementary approach to potentially reduce bias is to increase conditional entropy $H(T|\phibf)$ by adding \emph{randomization} to the selection policy $T$. Note that while this randomization increases $H(T|\phibf)$, it also increases $H(T)$ and thus could increase information usage. 
It is easy to maximize conditional entropy by choosing $T$ uniformly at random from $\{1,...,m\}$, independently of $\phibf$. Imagine however that we want to not only ensure that conditional entropy is large, but want to choose $T$ such that the selected value $\phi_T$ is large. After observing $\phibf$, it is natural then to set the probability $\pi_i$ of setting $T=i$ by solving a maximization problem
\begin{eqnarray*}
	\underset{{\bm \pi} \in \mathbb{R}^m_+}{\text{maximize}} && H({\bm \pi}) \\
	\text{subject to} && \sum_{i=1}^{k} \pi_i \phi_i \geq b \mbox{ and }  \sum_{i=1}^{k} \pi_i =1.
\end{eqnarray*}
The solution ${\bm \pi}^*$ to this problem is the maximum entropy or ``Gibbs'' distribution, which sets
\begin{equation}\label{eq: exponential weights}
\pi^*_i \propto e^{ \beta \phi_i} \qquad i \in \{1,..,m\}
\end{equation}
for $\beta>0$ that is chosen so that $\sum_i \pi^*_i \phi_i = b$.  This procedure effectively adds stability, or a kind of regularization, to the selection strategy by adding randomization. Whereas tiny perturbations to $\phibf$ may change the identity of $T=\arg\max_{i} \phi_i$, the distribution $\pi^*$ is relatively insensitive to small changes in $\phibf$. Note that the strategy \eqref{eq: exponential weights} is one of the most widely studied algorithms in the field of online learning \cite{cesa2006prediction}, where it is often called \emph{exponential weights}. It is also known as the exponential mechanism in differential privacy. In our framework it is transparent how it reduces  bias.

To illustrate the effect of randomized selection, we use simulations to explore the tradeoff between bias and accuracy. We consider the following simple, max-entropy randomization scheme:
\begin{itemize}
	\item Take as input parameters $\beta$ and $K$, and observations $\phi_1,...\phi_m$. Here $\beta$ is the inverse temperature in the Gibbs distribution and $K$ is number of $\phi_i$'s we need to select.
	\item Sample without replacement $K$ indices $T_1,...T_K$ from ${\bm \pi}^*$ given in \eqref{eq: exponential weights}. Report the corresponding values $\phi_{T_1}, ..., \phi_{T_k}$.
\end{itemize}
We consider settings where we have two groups of $\phi_i$'s: after relabeling assume that $\mu_1= ...=\mu_{N_1}=\mu > 0$ and $\mu_i = 0$ for $i > N_1$. We define the \emph{bias} of the selection to be $\frac{1}{K} \sum_{i=1}^K (\phi_{T_i} - \mu_{T_i}) $ and the \emph{accuracy} of the selection to be $|\{T_i: T_i \leq N_1 \}|/K$, which is the fraction of reported $\phi_{T_i}$ with true signal $\mu$. In Figure~\ref{fig:randomized_selection}, we illustrate the tradeoff between accuracy and bias for $N_1 = 1000, n-N_1 = 100000$ (i.e. there are many more false signals than true signals), randomization strength $\beta = 2$, and the signal strength $\mu$ varying from 1 to 5. Consistent with the theoretical analysis, max-entropy selection significantly decreased bias. In the low signal regime ($\mu = 1$), both rank selection and max-entropy selection have low accuracy because the signal is overwhelmed by the large number of false positives. In the high signal regime ($\mu \geq 4$), both selection methods have accuracy close to one and max-entropy selection has significantly less bias. In the intermediate regime ($ 1 < \mu < 4$), max-entropy selection has substantially less bias but is less accurate than rank selection.

Formally, unless the Gibbs distributions is degenerate with probability 1, 
\[
I(T;\phibf) = H(T)-H(T| \phibf) \leq \log(m) - H(T| \phibf) < \log(m), 
\]
so information usage is strictly smaller than its worst-case value of $\log(m)$. It is worth highlighting, however, that the Gibbs mechanism described above does not reduce bias or information usage for all possible data--generating distributions because it could increase entropy $H(T)$. 

\begin{figure}[!htb]
	\centering
	\includegraphics[ clip,scale=0.4]{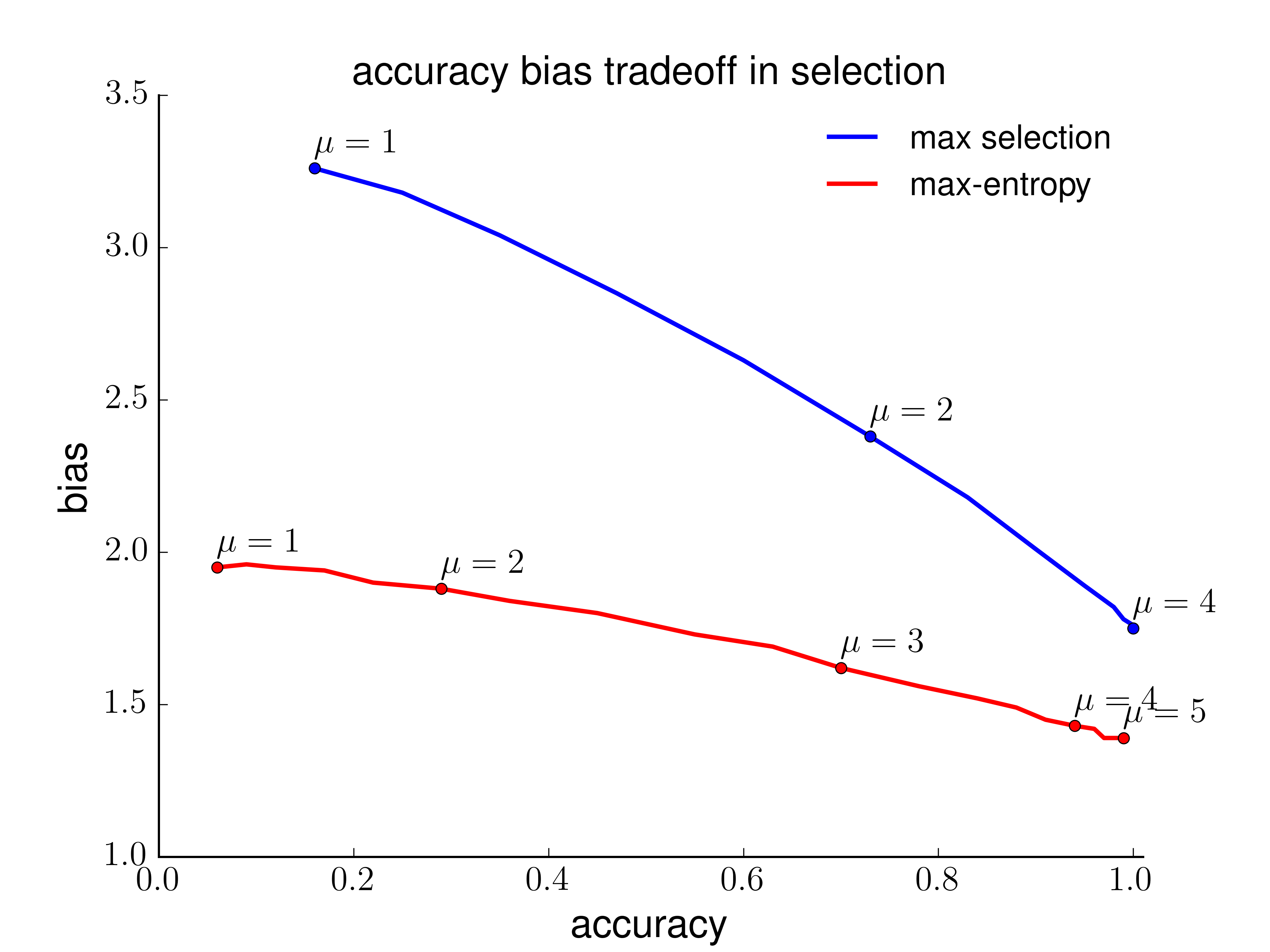}
	\caption{Tradeoff between accuracy and bias as the signal strength $\mu$ increases. The two curves illustrate the tradeoff for the maximum selection (i.e. reporting the largest $K = 100$ values of $\phi_i$) and the max-entropy randomized selection procedures.}
	\label{fig:randomized_selection}
\end{figure}

\subsection{Randomization for a multi-step analyst}

We next study how randomization can decrease information usage and bias even when we have very little knowledge of what the analyst is doing. To illustrate this idea,
we analyze in detail a simple example of a very flexible data analyst who performs multiple steps of analysis. Flexibility in multi-step data analysis presents a challenge to current statistical approaches for quantifying selection bias. Recent development in post-selection inference have focused on settings where the selection rule is simple and analytically tractable, and the full analysis procedure is fixed and specified before any data analysis is performed. While powerful results can be derived in this framework---including exact bias corrections and valid post-selection confidence intervals \cite{fithian2014optimal, Taylor2015}---these methods do not apply for exploratory analysis where the procedure can be quite flexible. 


In this section, we show how our mutual information framework can be used to analyze bias for a flexible multi-step analyst. We show that even if one does not know, or can't fully describe, the selection procedure $T$, one can control its bias by controlling the information it uses. The main idea is to inject a small amount of randomization at each step of the analysis. This randomization is guaranteed to keep the bad information usage low \emph{no matter what the analyst does.} 


The idea of adding randomization during data analysis to reduce overfitting has been implemented as practical rule-of-thumb in several communities. Particle physicists, for example, have advocated \emph{blind data analysis}: when deciding which results to report, the analyst interacts with a dataset that has been obfuscated through various means, such as adding noise to observations, removing some data points, or switching data-labels. The raw, uncorrupted, dataset is only used in computing the final reported values \cite{maccoun2015blind}.  Adding noise is also closely related to a recent line of work inspired by differential privacy \cite{blum2015ladder, dwork2015generalization,dwork2015reusable, hardt2014preventing}. 


\textbf{A model of flexible, multi-step analyst.}
We consider a model of adaptive data analysis similar to that of \cite{dwork2015reusable, dwork2015generalization}. In this setting, the analyst learns about the data by running a series of analyses on the dataset. Each analysis is modeled by a function of the data $\phi_i$, and choice of which analysis to run may depend on the results from all the earlier analyses.  More formally, we define the model as follows:

\begin{enumerate}
	\item At step 1, the analyst selects a statistic $\phi_{T_1}$ to query for $T_1 \in [m]$ and observes a result $Y_{T_1} \in \mathbb{R}$.
	\item In the $k$-th iteration, the analyst chooses a statistic $\phi_{T_k}$ as a function of the results that she has received so far, $\{Y_{T_1}, T_1, ..., Y_{T_{k-1}}, T_{k-1}\}$, and receives result $Y_{T_k}$.
	\item After $K$ iterations, the analyst selects $\phi_T \equiv \phi_{T_{K+1}}$ as a function of $\{Y_{T_1}, T_1, ..., Y_{T_{K}}, T_K\}$
\end{enumerate}

The simplest setting is when the result of the analysis is just the value of $\phi_{T_k}$ on the data $D$: $Y_{T_K} = \phi_{T_K}(D)$. An example of this is the rank selection considered before. At the $k$-th step, $\phi_k$ is queried (i.e. the order is fixed and does not depend on the previous results) and $Y_k = \phi_k$ is returned. The analyst queries all $m$ $\phi_i$'s and returns the one with maximal value. 

In general, we allow the analysis output $Y_{T_K}$ to differ from the empirical value of the test $\phi_{T_K}$ and a particularly useful form is $Y_{T_k} = \phi_{T_k} + \mbox{noise}$ . This captures blind analysis settings, where the analyst intentionally adds noise throughout the data analysis in order to reduce over-fitting. A natural goal is to ensure that for every query $T_k$ used in the adaptive analysis, the reported result $Y_{T_K}$ is close to true value $\mu_{T_K}$.  We will show through analyzing the information usage that noise addition can indeed guarantee such accuracy.   


This adaptive analysis protocol can be viewed as a Markov chain
\[
T_{k+1} \leftarrow H_{k} \equiv \{T_1, Y_{T_1}, ...,T_{k}, Y_{T_k} \} \leftarrow D \rightarrow \phibf,
\]
where recall that $\phibf$ denotes the vector $\{\phi_1, ..., \phi_m \}$. 
By the information processing inequality \cite{cover2012elements}, $I(T_{k+1}; \phibf) \leq I(H_k; \phibf)$. Therefore, a procedure that controls the mutual information between the history of feedback $H_k$ and the statistics $\phibf$ will automatically control the mutual information $I(T_{k+1}; \phibf)$. 
By exploiting the structure of the adaptive analysis model, we can decompose the cumulative mutual information $I(H_k; \phibf)$ into a sum of $k$ terms. This is formalized in the following \emph{composition} lemma for mutual information.

\begin{lem}\label{lem: composition}
	Let
	$
	H_{k} = \left(T_1, Y_{T_1}, T_2, Y_{T_2},..., T_k, Y_{T_k}\right)
	$
	denote the history of interaction up to time $k$. Then, under the adaptive analysis model
	\[
	I(T_{k+1} ; \phibf) \leq I(H_k ; \phibf) = \sum_{i=1}^{k} I(Y_{T_i} ; \phi_{T_i} | H_{i-1}, T_i )
	\]
\end{lem}

%

The important takeaway from this lemma is that by bounding the conditional mutual information between the response and the queried value at each step, $I(Y_{T_i} ; \phi_{T_i} | H_{i-1}, T_i )$, we can bound $I(T_{k+1}; \phi)$ and hence bound the bias after $k$ rounds of adaptive queries. Given a dataset $D$, we can imagine the analyst having a (mutual) \emph{information budget}, $I_b$, which is decided a priori based on the size of the data and her tolerance for bias. At each step of the adaptive data analysis, the analyst's choice of statistic to query next (as a function of her analysis history) incurs an information cost quantified by  $I(Y_{T_i} ; \phi_{T_i} | H_{i-1}, T_i )$. The information costs accumulate additively over the analysis steps, until it reaches $I_b$, at which point the guarantee on bias requires the analysis to stop.

A trivial way to reduce mutual information is to return a response $Y_{T_i}$ that is independent of the query $\phi_{T_i}$, in which case the analyst learns nothing about the data and incurs no bias. However in order for the data to be useful for the analyst, we would like the results of the queries to also be accurate. 

\textbf{Adding randomization to reduce bias.}
As before let $\mu_i = \E[\phi_i]$ denote the true answer of query $\phi_i$. If each $\phi_i - \mu_i$ is $\sigma$--sub-Gaussian, then $\E[|\phi_i - \mu_i|]\leq \sigma$. Using Proposition \ref{prop: absolute and squared error}, we can bound the average excess error of the response $Y_{T_k}$, $\E[|Y_{T_k} - \mu_{T_k}|]- \sigma,$ by the sum of two terms,  
\begin{align*}
&\E[|Y_{T_k} - \mu_{T_k}|] - \sigma \\
&\leq \E[|Y_{T_k} - \phi_{T_k}|]+ \E[|\phi_{T_k} - \mu_{T_k}|-\sigma] \\
&\leq \underbrace{\E[|Y_{T_k}-\phi_{T_k} |]}_{\text{Distortion}} + \underbrace{ c\sigma\sqrt{ 2I(T_k ; \phibf)}}_{\text{Selection Bias}}.
\end{align*}
Response accuracy degrades with distortion, a measure of the magnitude of the noise added to responses, but this distortion also controls the degree of selection bias in future rounds. We will explicitly analyze the tradeoff between these terms in a stylized case of the general model.

\textbf{Gaussian noise protocol.} We analyze the following special case.
\begin{enumerate}
	\item Suppose $\phi_i \sim \N(\mu_i, \frac{\sigma^2}{n})$ and $(\phi_1, ..., \phi_k)$ is jointly Gaussian for any $k$.
	\item For the $j$th query $\phi_{T_j}$, $j = 1, 2,...$, the protocol returns a distorted response
	$
	Y_{T_j} = \phi_{T_j} + W_{j}
	$
	where $W_j \sim \N(0, \frac{\omega_j^2}{n})$. Note that unlike $(\phi_1, \phi_2,....)$, the sequence $(W_1, W_2,....)$ is independent.
\end{enumerate}
The term $n$ can be thought of as the number of samples in the data-set. Indeed, if $\phi_i$ is the empirical average of $n$ samples from a $\N(\mu_i, \sigma^2)$ distribution, then $\phi_i \sim \N(\mu_i, \sigma^2/n)$. The ratio $\sigma^2 / \omega_j^2$ is the signal-to-noise ratio of the $k$th response. We want to choose the distortion levels $(\omega_1, \omega_2,...)$ so as to guarantee that a large number of queries can be answered accurately. In order to do this, we will use the next lemma to relate the distortion levels to the information provided by a response. The lemma gives a form for the mutual information $I(X ; X+W)$ where $X$ and $W$ are independent Gaussian random variables. As one would expect, this shows that mutual information is very small when the variance of $W$ is much larger than the variance of $X$. Lemma \ref{lem: variance bound on mutual info}, provided in the Appendix, provides a similar result when $X$ is a general (not necessarily Gaussian) random variable.
\begin{lem}\label{lem: lemma2}
	If $X \sim \N(0,\sigma_1^2)$ and $Y= X+W$ where $W\sim \N(0, \sigma_2^2)$ is independent of $X$, then
	\[
	I(X; Y) = \frac{1}{2}\log\left(1+\beta \right) \leq \frac{\beta}{2}
	\]
	where $\beta = \sigma_1^2/ \sigma_2^2$ is the signal to noise ratio.
\end{lem}
Using Lemma~\ref{lem: lemma2}, we provide an explicit bound on the accuracy of $Y_{T_{k+1}}$ as a function a function of $n, \sigma$ and $k$. Note that
this result places no restriction on the procedure that generates ($T_1, T_2,...$) except that the choice $T_k$ can depend on $\phibf$ only through the data $\{T_1, Y_{T_1},...T_{k-1}, Y_{T_{k-1}}\}$ available at time $k$.

\begin{prop} \label{prop: nsquared_short}
	Suppose $\phi_i \sim \N(\mu_i, \frac{\sigma^2}{n})$ and $(\phi_1, ..., \phi_k)$ is jointly Gaussian for any $k$. If for the $j$th query, $
	Y_{T_j} = \phi_{T_j} + W_{j}
	$
	where
	$W_j \sim \N(0, \frac{\sigma^2 \sqrt{j}}{n})$ and $(W_1, W_2, ...)$ is independent of $\phibf$,
	then for every $k\in \mathbb{N}$
	\[
	\E[|Y_{T_{k+1}} - \mu_{T_{k+1}} |] \leq c \left( \frac{\sigma k^{1/4}}{ n^{1/2}}      \right)
	\]
	where $c$ denote a universal constant that is independent of $\sigma, \omega, k, $ and $n$.
\end{prop}
If the sequence of choices ($T_1, T_2, T_3,...$) were non-adaptive, simply returning responses without any noise $(Y_{T_i}=\phi_{T_i}$) would guarantee $\E[|Y_{T_{k+1}} - \mu_{T_{k+1}} |] \leq \sigma/\sqrt{n}$. In the adaptive model, the first few queries are still answered with accuracy of order $\sigma/\sqrt{n}$, but the error increases for the later queries. This illustrates the fundamental tension that the longer the analyst explores the data, the more likely for the later analysis to overfit. 

The factor $k^{1/4}$ can roughly be viewed as the worst-case price of adaptivity. It is worth emphasizing this price would be more severe if the system returned responses without any noise. When no noise is added error can be as large as $\E [|Y_{T_{k+1}} -\mu_{T_{k+1}}|] =\Omega(\sigma \sqrt{k/n})$, as is demonstrated in Example~\ref{eg:overfittinglinear} in the Appendix. Therefore, adding noise offers a fundamental improvement in attainable performance.

A similar insight was attained by \cite{dwork2014preserving}, who noted that by adding Laplacian noise it is possible to answer up to $n^{2}$ queries accurately, whereas without noise accuracy degrades after $n$ queries. In the Gaussian case, it's clear from our bound that as $n,k\rightarrow \infty$, all queries will be answered accurately as long as $k=o(n^2)$.

\section{Discussion}
We have introduced a general information usage approach to quantify bias that arises from data exploration.
While we focus on bias, we show our mutual information based metric can be used to bound other error metrics of interest, such as the average absolute error $\E[|\phi_T - \mu_T|]$. It is interesting to note that the same information usage also naturally appears in the lower bound on error, suggesting it may be fundamentally linked to exploration bias. 
This paper established lower bounds when the selection process corresponds to solving optimization problems---i.e. $T = \arg\max$. An interesting direction of research is to understand more general exploration procedures in which information usage provides a tight approximation to bias. 

One advantage of using mutual information to bound bias is that we have many tools to analyze and compute mutual information. This conceptual framework allow us to extract insight into settings when common data analysis procedures lead to severe bias and when they do not. In particular we show how signal in the data can reduce selection bias. Information usage also suggests engineering approaches to reduce mutual information (and hence bias) by adding randomization to each step of the data exploration. Another important project is to investigate  implementations of such randomization approaches in practical analytic settings. 

As discussed before, the information usage framework proposed here is very much complementary to the exciting developments in post-selection inference and differential privacy. Post-selection inference, for very specific settings, is able to exactly characterize and correct for exploration biases---in this case exploration is feature and model selection. Differential privacy lies at the other extreme in that it derives powerful but potentially conservative results that apply to an adversarial data-analyst. The modern practice of data science often lies in between these two extremes---the analyst has more flexibility than assumed in post-selection inference, but is also interested in finding true signals and hence is much less adversarial than the worst-case. Information usage provides a bound on exploration bias in all settings. It is also important that this bound is data-dependent. In practice, the same analyst may be much less prone to false discoveries when exploring a high-signal dataset versus a low-signal dataset, and this should be reflected in the bias metric. An interesting goal is to develop approaches that combine the sharpness of post-selection inference and differential privacy with the generality of information usage.

\appendices
	\section{OVERVIEW OF THE APPENDIX} The appendix provides complete proofs of all the results in the main text as well as extensions and additional applications of information usage. Section~\ref{sec_app:upperbound} gives the proof of Proposition~\ref{prop: main result}, which states that information usage can be used to upper bounds selection bias. We also show that more general results hold when the estimators have different variances and when the estimators have heavier tales (i.e. sub-exponential rather than sub-Gaussian). Section~\ref{sec_app:lowerbound} then proves that the error due to exploration is at least as large as the information usage for several families of explorations, which includes Proposition~\ref{prop: lower bound}. Section~\ref{sec_app:overfitting} completes the proof of the link between information usage and classification overfitting (Proposition~\ref{prop: overfitting}). In Section~\ref{sec_app:otherapplications}, we provide additional applications to show how information usage can be used to control the bias in other metrics of interest, such as p-values in a multiple hypothesis testing problem and regret in optimization under uncertainty. Section~\ref{sec_app:expt_details} provides additional details of the experiments corresponding to Figure~\ref{fig:lars_example}. Section~\ref{sec_app:multistep} completes the analysis of how randomization controls the bias of a multi-step, flexibile data analyst. Section~\ref{sec_app:maxinfo} discusses how our information usage relates to other information measures such as max-information.  

\section{Proofs of Information Usage Upper Bounds} \label{sec_app:upperbound}

\subsection{Information Usage Upper Bounds Bias: Proof of Proposition \ref{prop: main result} }
The proof of Proposition \ref{prop: main result} relies on the following variational form of Kullback--Leibler divergence, which is given in Theorem 5.2.1 of Robert Gray's textbook \emph{Entropy and Information Theory} \cite{gray2011entropy}.
\begin{fact}\label{fact: variational defn of kl}
	Fix two probability measures $\mathbf{P}$ and $\mathbf{Q}$ defined on a common measurable space $(\Omega, \mathcal{F}).$ Suppose that $\mathbf{P}$ is absolutely continuous with respect to $\mathbf{Q}$. Then
	\[
	D\left( \mathbf{P}  || \mathbf{Q} \right) = \sup_{X} \left\{ \E_{\mathbf{P}}[ X] - \log \E_{\mathbf{Q}} [e^{ X } ]\right\},
	\]
	where the supremum is taken over all random variables $X$ such that the expectation of $X$ under $\mathbf{P}$ is well defined, and $e^{ X }$ is integrable under $\mathbf{Q}$.
\end{fact}

\begin{proof}[Proof of Proposition \ref{prop: main result}]
	\begin{eqnarray*}
		I(T; \phibf) &=&  \sum_{i=1}^{n} \Prob(T=i) D\left( \Prob(\phibf=\cdot | T=i) \, || \,  \Prob(\phibf=\cdot)  \right)  \\
		&\geq & \sum_{i=1}^{n} \Prob(T=i) D\left( \Prob(\phi_i=\cdot | T=i) \, || \,  \Prob(\phi_i=\cdot)  \right)
	\end{eqnarray*}
	Applying Fact \ref{fact: variational defn of kl} with $\mathbf{P} = \Prob(\phi_i=\cdot | T=i)$, \, $\mathbf{Q} = \Prob(\phi_i=\cdot)$, and $X=\lambda (\phi_i-\mu_i)$,  we have
	\[
	D\left( \Prob(\phi_i=\cdot | T=i) \, || \,  \Prob(\phi_i=\cdot)  \right) \geq \sup_{\lambda} \lambda \Delta_{i} - \lambda^2 \sigma^2/2
	\]
	where $\Delta_{i} \equiv \E[\phi_i | T = i] - \mu_i$. Taking the derivative with respect to $\lambda$, we find that the optimizer is $\lambda = \Delta_{i}/\sigma^2$. This gives
	\begin{eqnarray*}
		2\sigma^2 I(T; \phibf) &\geq& \sum_{i=1}^n \Prob(T=i) \Delta_i^2 = \E[\Delta_{T}^2].
	\end{eqnarray*}
	By the tower property of conditional expectation and Jensen's inequality
	\[
	\E[\phi_T - \mu_T] = \E[ \Delta_{T}] \leq \sqrt{\E[\Delta_{T}^2]} \leq \sigma \sqrt{2 I(T; \phibf)}.
	\]
\end{proof}

\noindent \textbf{Remark.} \emph{In the first step of the proof of Proposition~\ref{prop: main result}, we used the fact that, for all $i\in \{1,...,m\}$,  
	\begin{align*}
	& \quad D\left( \Prob(\phibf=\cdot | T=i) \, || \,  \Prob(\phibf=\cdot)  \right) \\
	\geq \quad & D\left( \Prob(\phi_i=\cdot | T=i) \, || \,  \Prob(\phi_i=\cdot)  \right), 
	\end{align*}
	which follows from the information processing inequality. The application of this inequality is not tight in general and can lead to gaps between the actual bias and our upper bound based on $I(T; \phibf)$. Consider the following scenario. Suppose $T: \phi_1 \rightarrow [2,...,m]$, i.e. $T$ is a deterministic function that uses the realized value of $\phi_1$ to decide which \emph{other} $\phi_j$ to select. For example, imagine $\phi_1 \sim {\rm Uniform}[0,1]$ and $T$ is defined so that $T=2$ if $\phi_1 \in [0,1/(m-1)]$, $T=3$ if $\phi_i \in [1/(m-1),2/(m-1)]$, and so on. Here $T$ is deterministic, $I(T; \phibf) = \log m$, and this is manifested in $D\left( \Prob(\phibf=\cdot | T=i) \, || \,  \Prob(\phibf=\cdot)  \right) > 0$. However, if $\phi_j, j\neq 1$ is independent of each other $\phi_i$, then $D\left( \Prob(\phi_i=\cdot | T=i) \, || \,  \Prob(\phi_i=\cdot)  \right) = 0$ and the bias is also 0. The upper bound of Proposition~\ref{prop: main result} is tight in other settings; it is also useful in general because the mutual information $I(T; \phibf)$ is amenable to analysis and explicit calculation. In  cases where there is a gap, we may study $\Prob(T=i) D\left( \Prob(\phi_i=\cdot | T=i) \, || \,  \Prob(\phi_i=\cdot)  \right)$ directly.}

\subsection{Extension to Unequal Variances}
We can prove a generalization of Proposition~\ref{prop: main result} for settings when the estimates $\phi_i$ have unequal variances.

\begin{prop}\label{prop: main result 2}
	Suppose that for each $i \in \{1,...,m\}$, $\phi_i - \mu_i$ is $\sigma_i$--sub-Gaussian. Then,
	\[
	|\E [\phi_T] - \E[\mu_T]| \leq \sqrt{\E[\sigma_T^2]}\sqrt{2 I(T; \phibf)}
	\]
	where $I$ denotes mutual information. 
\end{prop}

\begin{proof}
	The first part of the proof is the same as that of Proposition~\ref{prop: main result}. For each $i \in \{1,...,m\}$, 
	\[
	D\left( \Prob(\phi_i=\cdot | T=i) \, || \,  \Prob(\phi_i=\cdot)  \right) \geq \sup_{\lambda} \lambda \Delta_{i} - \lambda^2 \sigma_i^2/2
	\]
	where $\Delta_{i} \equiv \E[\phi_i | T = i] - \mu_i$. The optimizer is $\lambda_i = \Delta_{i}/\sigma_i^2$. Rearranging the terms gives
	\[
	\Delta_i \leq \sigma_i \sqrt{2 D\left( \Prob(\phi_i=\cdot | T=i) \, || \,  \Prob(\phi_i=\cdot)  \right)}.
	\]
	This implies
	\begin{align*}
	&\quad \E[\Delta_T]\\
	=&\quad \sum_i \Delta_i \Prob(T = i) \\
	\leq&\quad   \sum_i \sigma_i\Prob(T = i) \sqrt{2 D\left( \Prob(\phi_i=\cdot | T=i) \, || \,  \Prob(\phi_i=\cdot)  \right)}  \\
	\leq &\quad  \sqrt{ \sum_i \sigma_i^2 \Prob(\phi_i=\cdot | T=i) } \\
	&\quad \times \sqrt{ 2\sum_i \Prob(\phi_i=\cdot | T=i)  D\left( \Prob(\phi_i=\cdot | T=i) \, || \,  \Prob(\phi_i=\cdot)  \right)  }\\
	=&\quad  \sqrt{\E[\sigma_T^2]} \sqrt{2 I(T; \phibf)}.
	\end{align*}
	where we have used Cauchy-Schwartz for the second inequality.
	
\end{proof}

\subsection{Extension to Sub-exponential Random Variables}

Recall that a random variable $X$ is sub-Gaussian with parameter $\sigma$ if $\E[e^{\lambda(X-\E[X])}] \leq e^{\lambda^2 \sigma^2 /2}$ for \emph{all} real-values $\lambda$. While many random variables are sub-Gaussian, there are other important classes of random variables that are light tailed, but not quite sub-Gaussian. Here, we will show how our information-usage bounds extend to the larger class of sub-exponential random variables. We say that $X$ is sub-exponential with parameters $(\sigma,b)$ if $\E[e^{\lambda(X-\E[X])}] \leq e^{\lambda^2 \sigma^2 /2}$ whenever $|\lambda| < 1/b$. For example if $X\sim\chi^2_n$ follows a chi-squared distribution with $n\geq 1$ degrees of freedom, then it is sub-exponential with parameters $(2\sqrt{n},4)$.
\begin{prop}\label{prop: subexponential}
	Suppose that for each $i \in \{1, ..., m\}$, $\phi_i - \mu_i$ is sub-exponential with parameters $(\sigma, b)$. Then
	\[
	\E[\phi_T - \mu_T] \leq bI(T; \phibf) + \frac{\sigma^2}{2b}.
	\]
	Moreover, if $b < 1$, we also have
	\[
	\E[\phi_T - \mu_T] \leq \sqrt{b}I(T; \phibf) + \frac{\sigma^2}{2\sqrt{b}}.
	\]
\end{prop}
\begin{proof}
	Following the same analysis as in the sub-Gaussian setting (Prop.~\ref{prop: main result}), we have
	\[
	D\left( \Prob(\phi_i=\cdot | T=i) \, || \,  \Prob(\phi_i=\cdot)  \right) \geq \sup_{\lambda < 1/b} \lambda \Delta_{i} - \lambda^2 \sigma^2/2
	\]
	
	The RHS is greater than the value from setting $\lambda = 1/b$. Therefore, we have
	\[
	D\left( \Prob(\phi_i=\cdot | T=i) \, || \,  \Prob(\phi_i=\cdot)  \right) \geq \frac{\Delta_{i}}{b} - \frac{\sigma^2}{2b^2}.
	\]
	Multiplying each side by $P(T = i)$ and summing over $i\in \{1,..,m\}$ gives
	\[
	I(T; \phibf) \geq \frac{\E[\phi_T - \mu_T ]}{b} - \frac{\sigma^2}{2b^2}
	\]
	and hence
	\[
	\E[\phi_T - \mu_T] \leq bI(T; \phibf) + \frac{\sigma^2}{2b}.
	\]
	When $b < 1$, $\lambda = 1/\sqrt{b} < 1/b$ is also a feasible point. Putting in this value of $\lambda$ into the calculations above gives the second bound
	\[
	\E[\phi_T - \mu_T] \leq \sqrt{b}I(T; \phibf) + \frac{\sigma^2}{2\sqrt{b}}.
	\]
\end{proof}

\subsection{Extension to Other Metrics of Exploration Error}
\begin{propn}[\ref{prop: absolute and squared error} - Part (1)]
	Suppose for each $i \in \{1,...,m\}$, $\phi_i - \mu_i$ is $\sigma$ sub-Gaussian. Then
	\[
	\E[| \phi_T - \mu_T |] \leq \sigma + c\cdot \sigma \sqrt{2 I(T; \phibf)}
	\]
	where $c < 36$ is a universal constant.
\end{propn}

\begin{proof}
	Let $U_i = \phi_i - \mu_i$ which is assumed to be $\sigma$ sub-Gaussian and let $\gamma_i = \E[|\phi_i - \mu_i |]$ and $Y_i = |U_i| - \gamma_i$. We show below that $Y_i$ is sub-Gaussian with parameter $c\sigma$ where $c\leq 36$. This implies the result, since by Proposition~\ref{prop: main result} and the data-processing inequality,
	\[
	\E[| \phi_T - \mu_T | - \gamma_T ] = \E[Y_T] \leq c \sigma \sqrt{2 I (T; \Ybf)}\leq \sqrt{2 I (T; \phibf)}.
	\]
	Since $\gamma_i \leq \sigma$ for all $i$, $\gamma_T \leq \sigma$, and we have
	\[
	\E[| \phi_T - \mu_T | ] \leq \sigma + 36 \sigma \sqrt{2 I (T; \phibf)}.
	\]

	The remainder of the proof shows $Y\equiv|U|-\E[|U|]$ is sub-Gaussian whenever $U$ is sub-Gaussian. We use the following equivalent definition of a sub-Gaussian random variable.
	
	\paragraph{Fact 1.} \cite{wainwright15} Given a zero-mean random variable $Y$, Suppose there is a constant $c \geq 1$ and Gaussian random variable $Z \sim \N(0, \tau^2)$ such that
	\[
	\Prob(|  Y | \geq s) \leq c( \Prob (|Z| \geq s)) \mbox{ for all } s \geq 0.
	\]
	Then $Y$ is sub-Gaussian with parameter $\sqrt{2} c \tau$.
	
	\paragraph{Fact 2.} \cite{wainwright15} Suppose $Y$ is a zero-mean sub-Gaussian random variable with parameter $\sigma$. Then
	\[
	\Prob(| Y | \geq s) \leq \sqrt{8}e \Prob (|Z| \geq s)
	\]
	where $Z \sim \N(0, 2\sigma^2)$.
	
	Let $U$ be a zero-mean random variable that is sub-Gaussian with parameter $\sigma$.  Let $\gamma \equiv \E[|U|]$ and $Y \equiv | U | - \gamma$. We want to determine the sub-Gaussian parameter of $Y$.
	We have
	\begin{eqnarray*}
		\Prob(|Y| \geq s) &=& \Prob(| U | \geq s + \gamma) + \Prob(| U | \leq \gamma - s) \\
		&\leq & \sqrt{8} e \Prob( |Z| \geq s) + \Prob(| U | \leq \gamma - s)
	\end{eqnarray*}
	where $Z \sim \N(0, 2\sigma^2)$ and we have used Fact 2. Moreover
	\[
	\Prob(| U | \leq \gamma - s) \leq \frac{\Prob( |Z| \geq s )}{ \Prob(|Z| \geq \gamma) }
	\]
	since the RHS exceeds $1$ for $s \leq \gamma$ and the LHS is 0 for $s > \gamma$. Hence
	\[
	\Prob(|Y| \geq s) \leq \left( \sqrt{8}e + \frac{1}{\Prob(|Z| \geq \gamma)} \right) \Prob(|Z| \geq s)
	\]
	and, by Fact 1, $Y$ is sub-Gaussian with parameter $2(\sqrt{8}e + 1/\Prob(|Z| \geq \gamma))\sigma$. We can simplify this expression further. Since $U$ is $\sigma$ sub-Gaussian, its variance is bounded above by $\sigma^2$. Therefore $\gamma \leq \sqrt{\E[U^2]} \leq \sigma$, which implies
	\[
	\Prob(|Z| \geq \gamma) > \Prob(|Z| \geq \sigma) > 0.1
	\]
	and $Y$ is sub-Gaussian with parameter $36 \sigma$.
\end{proof}

This bound is similar to a bias-variance decomposition, where the $\sigma$ term is the variance and the mutual--information term is the bias. When selection is over many $\phi_i$'s, the bias term tends to dominate. The parameter $\sigma$ captures the magnitude of noise in the estimates, and therefore implicitly captures the number of samples in the data set. In particular, If $\phi_i  = n^{-1} \sum_{j=1}^{n} f_i (X_j)$ where $\{f_i(X_j)\}_{j=1}^{n}$ is an independent sequence of $\sigma$-sub-Gaussian random variables, then
\[
\E[| \phi_T - \mu_T |] \leq \frac{\sigma}{\sqrt{n}} + c\cdot \frac{\sigma}{\sqrt{n}} \sqrt{2 I(T; \phibf)}.
\]

Using the fact that the square of a sub-Gaussian random variable is sub-exponential and Proposition \ref{prop: subexponential}, we can also control the mean squared distance between $\phi_{T}$ and $\mu_{T}$.

\begin{propn}[\ref{prop: absolute and squared error} - Part (2)]
	Suppose $\phi_i - \mu_i$ is $\sigma$ sub-Gaussian for each $i \in \{1,...,m\}$. Then
	\[
	\E[ (\phi_T - \mu_T)^2]  \leq \sigma^2\left(1.25 + 10I(T; \phibf) \right).
	\]
\end{propn}

\begin{proof}
	We use the following fact about sub-Gaussian random variables.
	\paragraph{Fact 3.} \cite{wainwright15} If $Y$ be a zero-mean sub-Gaussian variable with parameter $\sigma$, then
	\[
	\E\left[e^{\frac{\lambda Y^2}{2\sigma^2}}  \right] \leq \frac{1}{\sqrt{1- \lambda}} \mbox{ for all } \lambda \in [0,1).
	\]
	Given such a $Y$, we would like to derive the sub-exponential parameters of $Y^2 - \gamma$, where $\gamma \equiv \E[Y^2] \geq 0$. Applying Fact 3, we have
	\[
	\E\left[e^{\frac{\lambda (Y^2 - \gamma)}{2\sigma^2}}  \right] \leq  \frac{1}{\sqrt{1- \lambda}} \leq e^{10\lambda^2} \mbox{ for } \lambda \in [0,0.1)
	\]
	where the last inequality can be verified numerically. Using the substitution $t \equiv \lambda/\sigma^2$, we have
	\[
	\E\left[e^{t(Y^2 - \gamma)}  \right] \leq e^{10 \sigma^4 t^2} \mbox{ for } t \in \left[0, \frac{0.1}{\sigma^2} \right)
	\]
	which implies that $Y^2 - \gamma$ is sub-exponential with parameters $(\sqrt{5} \sigma^2, 10\sigma^2)$.
	
	In our setting, $Y_i = \phi_i - \mu_i$ is $\sigma$ sub-Gaussian and $\gamma_i = \E[(\phi_i - \mu_i)^2] \leq \sigma^2$. Applying Proposition~\ref{prop: subexponential} to $Y_i^2$, we have
	\begin{align*}
	\E[ (\phi_T - \mu_T)^2] &\leq \sigma^2 + 10\sigma^2  I(T; \Ybf^2) + \frac{\sigma^2}{4}\\
	&\leq  \sigma^2\left(1.25 + 10I(T; \Ybf^2) \right) \\
	&\leq   \sigma^2\left(1.25 + 10I(T; \phibf)\right).
	\end{align*}
	where $\Ybf^2 \equiv (Y_1^2,...Y_m^2)$ and  the final step uses the data-processing inequality.
\end{proof}


In the next result, we think of $\phibf=(\phi_1,..,\phi_m)$ and $T$ as a collection of estimates and a choice of which one to report made based on \emph{common} data-set $D$, while we think of $\tilde{\phibf}=(\tilde{\phi}_1,...,\tilde{\phi}_m)$ as these same estimates computed on a fresh replication data-set $\tilde{D}$. The next result bounds the KL-divergence between $\phi_T$ and $\tilde{\phi}_T$, which captures the change in the distribution of the reported result due to performing selection and estimation on a common data-set. 
\begin{prop}\label{prop: kl form of main result}
	Let $\tilde{\phibf}$ denote a random variable drawn from the marginal distribution of $\phibf$, but drawn independently of $T$ and $\phibf$.  Then
	\[
	D\left( \Prob(\phi_{T} = \cdot)  \, || \,  \Prob(\tilde{\phi}_{T} =\cdot)  \right) \leq I\left(  T; \phibf \right).
	\]
\end{prop}

\begin{proof}
	\begin{align*}
	&\,\, D\left( \Prob(\phi_{T} = \cdot)  \, || \,  \Prob(\tilde{\phi}_{T} =\cdot)  \right)\\
	\leq&\,\,  D\left( \Prob(\phi_{T} = \cdot, T= \cdot)  \, || \,  \Prob(\tilde{\phi}_{T} =\cdot, T=\cdot)  \right) \\
	\leq&\,\, \sum_{T=1}^{m}\Prob(T=i)D\left( \Prob(\phi_{T} = \cdot | T=i)  \, || \,  \Prob(\tilde{\phi}_{T} =\cdot | T=i) \right)\\
	=&\,\, \sum_{T=1}^{m}\Prob(T=i)D\left( \Prob(\phi_{i} = \cdot | T=i)  \, || \,  \Prob(\phi_{i} =\cdot)\right) \\
	\leq&\,\,  \sum_{T=1}^{m}\Prob(T=i)D\left( \Prob(\phibf = \cdot | T=i)  \, || \,  \Prob(\phibf  =\cdot)\right) \\
	=&\,\, I(T ; \phibf ),
	\end{align*}
	where both inequalities follow from the data-processing inequality for KL divergence.
\end{proof}

\section{Information Usage Also Lower Bounds Bias} \label{sec_app:lowerbound}

\subsection{Top-k selection: a lower bound for Corollary~\ref{cor: orderstat} } 
Here we show that the bound of Corollary~\ref{cor: orderstat} is tight as $m/m_0 \rightarrow \infty$. For convenience, we show this when $m$ is divisible by $m_0$. Consider the following alternative selection policy $\hat{T}$.  Randomly partition the $\phi_i$'s into $m_0$ groups of size $m/m_0$. Within each group, select the maximal $\phi_i$ and from these $m_0$ maximal $\phi_i$'s randomly select one as $\phi_{\tilde{T}}$. Because the average among the $m_0$ group leaders is less than the average among the $\phi_{(1)}, ..., \phi_{(m_0)}$, we have $\E[\phi_{\tilde{T}}] \leq \E[\phi_{T}]$. Moreover, each group leader converges to $\sigma \sqrt{2 \log m/m_0}$ and since the groups are independent, the average $\E[\phi_{\tilde{T}}]$ also converges to $\sigma \sqrt{2 \log m/m_0}$.


\subsection{Maximum of Gaussians: Proof of Proposition \ref{prop: lower bound}}

Recall the statement of Proposition \ref{prop: lower bound}.
\begin{propn}[\ref{prop: lower bound}] Let $T=\arg\max_{1\leq i \leq m} \phi_i$ where $\phibf \sim \mathcal{N}(\mubf, I)$. There exist universal numerical constants $c_1=1/8$,  $c_2 < 2.5$ , $c_3=10$, and $c_4=1.5$ such that for any $m\in \mathbb{N}$ and $\mubf \in \mathbb{R}^{m}$,   
	\[
	c_1 H(T)- c_2 \leq \E[(\phi_T - \mu_T)^2]  \leq c_3H(T) + c_4 . 
	\] 
\end{propn}
The upper bound above follows by Proposition \ref{prop: absolute and squared error}. Here we will focus on establishing the lower bound. 

Throughout, we will use the notation $M \triangleq \phi_T= \max_{i} \phi_i$ and $M_{-i} \triangleq \max_{j\neq i} \phi_j$. We rely on the following facts. The first shows that the maximum of Gaussian random variables is itself a sub-Gaussian random variable. The second establishes a tail bound for normal random variables.

\begin{fact}
	$M\triangleq \max_{i} \phi_i$ is 1-subgaussian. In particular,
	$
	\E[ e^{\lambda(M -\E[M]}] \leq e^{\lambda^2/2}.
	$
	This implies the variance bound $\E[(M-\E[M])]^2 \leq 1$
	and the tail bounds $\Prob(M  \geq E[M] + \lambda) \leq e^{-\lambda^2/2}$. Similarly, $M_{-i}$ is 1-sub--Gaussian for all $i$. 
\end{fact}
\begin{fact}
	If $X\sim \mathcal{N}(0,1)$ then for all $x>0$
	\[
	\Prob(X > x) \geq \frac{1}{\sqrt{2\pi}}\left(\frac{x}{x^2+1}  \right) e^{-x^2 /2}
	\]
\end{fact}

Proposition \ref{prop: lower bound} provides an analogous lower bound. To understand this result, recall that entropy the entropy of $T$ is 
\[
H(T) = \sum_{i}\Prob(T=i) \log(1/ \Prob(T=i)).
\]
Consider a setting where $\E[M]$ significantly exceeds $\mu_i$. Then, since $M$ concentrates around $\E[M]$, the probability $i$ is maximal is close to the probability $\phi_i$ exceeds $\E[M]$. By the above fact, one expects that $\log(1/\Prob(T=i)) \approx \log\Prob(\phi_i > \E[M]) \approx (\E[M]-\mu_i)^2/2$. This is roughly the intuition behind the following result. Along with our upper bound, this describes a natural family of problems in which $\E[(\phi_T - \mu_T)^2] = \Theta(1+H(T))$.

\begin{proof}
	We focus on establishing the lower bound, as the upper bound follows from Proposition \ref{prop: absolute and squared error}.

	By definition, $T=i$ if and only if $M_{-i} \leq \phi_i$. Our proof will separately consider two cases, depending on whether $\E[M_{-i}] \geq \mu_i +1$. Let $I \equiv \{i: \E[M_{-i}] \geq \mu_i + 1\}$ denote the set of estimates whose mean is at least a full standard deviation below that of $M_{-i}$.

	
	The entropy of $T$ can be decomposed as 
	\begin{align*} 
	H(T) =& \sum_{i\notin I}\Prob(T=i) \log\left( \frac{1}{\Prob(T=i)} \right) \\
	&+ \sum_{i\in I}\Prob(T=i) \log\left( \frac{1}{\Prob(T=i)} \right).  
	\end{align*}
	We first upper bound the sum over $i \notin I$. We do this by lower bounding $\Prob(T=i)$, which yields an upper bound on $\log(1/\Prob(T=i))$. For any constant $\lambda > 0$, and $i \notin I$, $\Prob(M_{-i} < \E[M_{-i}]+\lambda)>1-e^{-\lambda^2/2}$. Using the fact that $\E[M_{-i}] < \mu_i + 1 $,  we have for all $\lambda\geq 0$
	\begin{eqnarray*}
		\Prob(T=i)&=& \Prob(M_{-i} < \phi_i)  \\
		&\geq& \Prob(M_{-i} < \E[M_{-i}]+\lambda) \cdot \Prob(\phi_i> \E[M_{-i}]+\lambda)\\
		&\geq& \Prob(M_{-i} < \E[M_{-i}]+\lambda)\cdot \Prob(\phi_i>\mu_i + 1 + \lambda) \\
		&\geq& \left(1 - e^{-\lambda^2/2}\right)\frac{1}{\sqrt{2\pi}}\left( \frac{1+\lambda}{(1+\lambda)^2+1}\right)e^{-(1+\lambda)^2/2} \\
		&\triangleq& p(\lambda).
	\end{eqnarray*}
	Therefore
	\begin{align*}
	&\quad \sum_{i \not \in I} \Prob(T=i) \log\left(\frac{1}{\Prob(T=i)}\right) \\
	\leq&\quad \Prob(T \not\in I)\max_{i \notin I} \log\left(\frac{1}{\Prob(T=i)}\right)\\ 
	\leq&\quad \log\left(\frac{1}{p(1)}\right) \triangleq c_{-I}.
	\end{align*}
	Direct calculation shows $c_{-I} < 5$. 
	
	Now we consider the case $i \in I$. 
	To simplify notation, consider the shifted random variables $X \equiv \phi_i - \mu_i \sim \mathcal{N}(0,1)$ and $Y \equiv M_{-i} - \mu_i$. We lower bound $\log(1/\Prob(T=i))$ by a function of $\E[Y]^2$. We have
	\begin{align*}
	&\quad \Prob(T=i)\\
	=&\quad \intop_{-\infty}^{\infty} \Prob(X > x) \Prob(Y=dx) \\
	\geq&\quad \intop_{1}^{\infty} \Prob(X > x) \Prob(Y=dx) \\
	=&\quad \Prob(Y\geq 1) \intop_{1}^{\infty} \Prob(X > x) \Prob(Y=dx| Y\geq 1)\\
	\geq&\quad  \frac{\Prob(Y\geq 1)}{\sqrt{2\pi}}\intop_{1}^{\infty} \left(\frac{x}{x^2+1}  \right) e^{-x^2 /2}  \Prob(Y=dx | Y\geq 1).
	\end{align*}
	By Jensen's inequality,
	\begin{align*}
	\log \Prob(T=i)  &\geq \quad \log(1/\sqrt{2\pi}) + \log(\Prob(Y\geq 1)) \\
	&\quad +\intop_{1}^{\infty} \left( \log \left(\frac{x}{x^2+1}\right) -x^2 /2 \right) \\
	&\qquad\qquad  \times  \Prob(Y=dx | Y\geq 1),
	\end{align*}
	which can be rewritten as
	\begin{align*}
	\log\left(\frac{1}{\Prob(T=i)} \right) \leq& \log(\sqrt{2\pi})+ \log\left(\frac{1}{\Prob(Y\geq 1)} \right)\\
	&+ \E\left[ \log\left( \frac{Y^2+1}{Y} \right) | Y>1 \right]\\
	&+ \frac{\E[Y^2 |Y>1 ]}{2}.
	\end{align*}
	For $Y\geq 1$, one has $\log ((Y^2+1)/Y) \leq \log(1+Y)\leq Y$. Therefore,
	\begin{align*}
	\log\left(\frac{1}{\Prob(T=i)} \right) \leq \log(\sqrt{2\pi}) &+ \log\left(\frac{1}{\Prob(Y\geq 1)}\right)\\
	&+ 1.5 \E[Y^2 |Y>1 ].
	\end{align*}
	Now,
	\begin{align*}
	\E[Y^2 |Y>1 ] &\leq \E[Y^2]/\Prob(Y>1) \\
	&= \left(\E[(Y- \E[Y])^2] + \E[Y]^2\right)/\Prob(Y>1).
	\end{align*}
	Since $Y= M_{-i} - \mu_i$, the variance of $Y$ is bounded by 1. Using as well that $\Prob(Y>1) \geq 1-1/\sqrt{e}$ gives the bound
	\begin{align*}
	\log\left(\frac{1}{\Prob(T=i)} \right) \leq& \log(\sqrt{2\pi})+ \log\left(\frac{1}{\Prob(Y\geq 1)}\right)\\
	&+\frac{1.5(1+ \E[Y]^2)}{\Prob(Y \geq 1)} \\  
	<& 5 + 4\E[Y]^2 .
	\end{align*}
	
	Now, plugging in $\E[Y] = \E[M_{-i}] - \mu_i$ and putting everything together, we find
	\begin{align*}
	H(T)&= \sum_{i}\Prob(T=i)\log(1/\Prob(T=i))\\
	&\leq c_{-I}+ 5 + 4\sum_{i\in I}\Prob(T=i) (\E[M_{-i}] - \mu_i)^2 \\
	&\leq c_{-I} + 5 +4\sum_{i \in I}\Prob(T=i) (\E[M] - \mu_i)^2 \\
	&\leq c_{-I} + 5 + 4 \|  \E[M] - \mu_T \|^2
	\end{align*}
	where $\| X\| \equiv \sqrt{\E[X^2]}$ denotes the $L_2$ norm a random variable $X$ and the second inequality uses that $\E[M] \geq \E[M_{-i}] \geq \mu_i$.
	
	We complete the proof by relating $\| \E[M] - \mu_T \| $  to $\| \phi_T - \mu_T\|$. Recall that $\phi_T$ is 1-sub--Gaussian and  $\E[\phi_T] = \E[M]$. Therefore 
	\[
	\| \E[M] - \phi_T\| =  \E[ (\phi_T - \E[\phi_T])^2] \leq 1.
	\]
	
	Combining this with the triangle inequality shows
	\[
	\|  \E[M] - \mu_T \| =  \|  \E[M]-\phi_T + \phi_T - \mu_T \| \leq 1 +\| \phi_T - \mu_T \|.
	\]
	
	We can then conclude
	\[
	\|  \E[M] - \mu_T \|^2 \leq  (1 +\| \phi_T - \mu_T \|)^2 \leq 2+ 2\| \phi_T - \mu_T \|^2
	\]
	where the inequality uses that $\max_{x\in \mathbb{R}} f(x) =0$ for $f(x) \equiv (1+x)^2 -2- 2 x^2$. Together, this shows
	\[
	H(T) \leq c_{-I} + 5+ 8 + 8\| \phi_T - \mu_T \|^2
	\]
	or
	\[
	\| \phi_T - \mu_T \|^2 \geq c_1 H(T)- c_2
	\]
	where  $c_1 = 1/8$ and $c_2 = (c_{-I}+13)/8<2.5$.
\end{proof}


\subsection{Threshold Selection with Gaussian Random Variables}

In addition to the max-selection policy, we analyze a softer threshold selection policy and prove that the information usage lower bounds bias here as well. Let each $\phi_i$ correspond to a Gaussian of variance 1, and we allow the Gaussians to have different means and be correlated.

Let $M$ be a constant. The threshold-$M$ selection procedure does the following:
\begin{enumerate}
	\item If at least one $\phi_i$ is larger than $M$, uniformly randomly select one of these $\phi_i$'s to report. For this, we exclude $\phi_{-1}$.
	\item Otherwise, always report an arbitrary, fixed $\phi_{-1}$.
\end{enumerate}

In what follows, we will show that for $M$ \emph{sufficiently large}, the entropy $H(T)$ lower bounds the square-loss bias $\E[(Z_T - \mu_T)^2]$, where, recall that $Z_T = \E[\phi_i | T = i]$. Let $N_{-i} = |\{\phi_j \geq M , j\neq i, j\neq -1\}|$. As $M$ increases, $\E[N_{-i} | \phi_i \geq M]$ decreases. We want the threshold to be high enough so that only a few $\phi_i$'s are expected to pass the threshold. Let $\hat{N}(M) = \max_i \E[N_{-i} | \phi_i \geq M]$. 

\begin{theorem}
	Suppose 
	\begin{align*}
	&\quad M - \max_i \mu_i \\
	\geq&\quad \sqrt{2\log[2\pi\left(1+\E\left[N_{-i} | \phi_i \geq M\right]\right)(M - \max_i \mu_i)] + 3},
	\end{align*}
	then 
	\[
	\E[(\phi_T - \mu_T)^2] \geq H(T).
	\]
\end{theorem} 

\begin{proof}
	
	For $i \neq -1$, define $p_i = \Prob(T = i)$. Then we have
	\begin{align*}
	p_i &= \Prob(\phi_i \geq M) \sum_{k = 0}^{n-1} \Prob(N_{-i} = k \big| \phi_i \geq M) \frac{1}{k+1} \\
	&= \Prob(\phi_i \geq M) \E\left[\frac{1}{1+N_{-i}} | \phi_i \geq M\right].
	\end{align*}
	
	Let $p = \sum p_i$ denote the probability that at least one $\phi_i, i\neq -1,$ passes the threshold. Note that here and below, when we write $\sum p_i$, we always mean the sum of over $i\neq -1$.  We can write the entropy as
	\begin{align*}
	H(T) =& \sum p_i \log \frac{1}{p_i} + (1-p)\log\frac{1}{1-p}\\
	=& \sum p_i \log \frac{1}{\Prob(\phi_i \geq M)} \\
	&+ \sum p_i \log\left(1/ \E\left[\frac{1}{1+N_{-i}} \big| \phi_i \geq M\right]\right)\\
	&+ (1-p)\log\frac{1}{1-p} \\
	\le&  \sum p_i \log \frac{1}{\Prob(\phi_i \geq M)}\\
	&+ \sum p_i \log\left(1+\E\left[N_{-i} | \phi_i \geq M\right]\right)\\
	&+ (1-p)\log\frac{1}{1-p}  \\
	\le&  \sum p_i \log \frac{1}{\Prob(\phi_i \geq M)}\\
	&+ \sum p_i \log\left(1+\E\left[N_{-i} | \phi_i \geq M\right]\right)+ p.
	\end{align*}
	
	We can rewrite the inequality as
	\begin{align*}
	\sum p_i \log \frac{1}{\Prob(\phi_i \geq M)} \geq& H(T)\\
	&- \sum p_i \log\left(1+\E\left[N_{-i} | \phi_i \geq M\right]\right)\\
	&- p.
	\end{align*}

	Since $\phi_i \sim \mathcal{N}(\mu_i,1)$ and $M > \mu_i$, we have the bounds
	\begin{align*}
	\frac{(M-\mu_i)^2}{2} \geq \log \frac{1}{\Prob(\phi_i \geq M)} - \log(M-\mu_i) &- \frac{1}{2}\log(2\pi) \\
	&- \frac{1}{(M-\mu_i)^{2}}.
	\end{align*}

	After some algebra we have 
	\begin{align*}
	&\E[(\phi_T - \mu_T)^2] \\
	\geq& \sum p_i (M - \mu_i)^2 \\
	\geq&  \sum p_i \bigg[\frac{(M-\mu_i)^2}{2} -\log\left(1+\E\left[N_{-i} | \phi_i \geq M\right]\right) \\
	&\quad\quad\quad  - \frac{\log2\pi}{2} - \log(M- \mu_i) - \frac{1}{(M-\mu_i)^2 }-1 \bigg]   \\
	&  +  H(T)\\
	\geq& H(T) 
	\end{align*}
	where the second inequality used the above inequalities for $\frac{(M-\mu_i)^2}{2} $ and $\sum p_i \log \frac{1}{\Prob(\phi_i \geq M)}$; and the third inequality used the condition that $M - \max_i \mu_i$ exceeds $\sqrt{2\log[2\pi\left(1+\E\left[N_{-i} | \phi_i \geq M\right]\right)(M - \max_i \mu_i)] +3}$.
	

\end{proof}

As $M$ increases, unless the $\phi_i$'s are very highly correlated, $\E[N_{-i} | \phi_i \geq M])$ decreases and $H(T)$ dominates in the inequality. This shows that $H(T)$ is a natural lower bound on $\E(Z_T^2)$ and hence $\sqrt{H(T)}$ lower bounds bias. Actually we can improve this lower bound by considering $I(T | \Phi) = H(T) - H(T|\Phi)$ using the fact that
\begin{eqnarray*}
	H(T|\Phi) &=& \sum_{N =1} \Prob(N = i) H(T | N = i) \\
	&=& \sum_{N = 1} \Prob(N = i) \log i \\
	&=& \E[\log N | N \geq 1]
\end{eqnarray*}
where $N = |\{\phi_i, \phi_i \geq M\}|$. Assuming that $\phi_i$'s are independent, we need to control the gap between $\E[Z_T^2]$ and $I(T, \Phi)$, we need to upper bound $p \log (1 + \E[N]) - \E[\log N | N \geq 1]$.

\subsection{Threshold Selection with Exponential Random Variables} 

We can prove the analogous lower bound for the threshold policy with exponential random variables. Let $\phi_i = \lambda_i + \exp(1)$ be the shifted exponential random variable. So for $x\geq \lambda_i$, $\Prob(\phi_i = x) = e^{-(x - \lambda_i)}$ and $\Prob(\phi_i = x) = 0 $ for $x < \lambda_i$. Different $\phi_i$'s can have different $\lambda_i$ and we allow them to be correlated. The mean of $\phi_i$ is $\mu_i = \lambda_i + 1$. As before, let $\hat{N}(M) = \max_i \E[N_{-i} | \phi_i \geq M]$.

\begin{theorem}
	Suppose $M - \max \lambda_i \geq 4+2\log (1+\hat{N}(M))$,
	\[
	\E[\phi_T - \mu_T] \geq H(T)/2.
	\]
\end{theorem}

\begin{proof}
	The proof follows the same structure as before. Since $\Prob(\phi_i > M) = e^{-(M-\lambda_i)}$, we have $\log 1/\Prob(\phi_i > M) = M-\lambda_i$ and 
	\begin{align*}
	H(T) &= \sum_{i \neq -1} p_i (M - \lambda_i)\\
	&+ \sum_{i \neq -1} p_i \log\left(1/ \E\left[\frac{1}{1+N_{-i}} \big| \phi_i \geq M\right]\right)\\
	&+ (1-p)\log\frac{1}{1-p}.
	\end{align*}
	
	On the other hand,
	\begin{eqnarray*}
		&&\E[\phi_T - \mu_T] \\
		&\geq& \sum_{i \neq -1} p_i (M-\mu_i) \\
		&=& \sum_{i \neq -1} p_i (M-\lambda_i) - p \\
		&\geq& H(T) - \sum_{i \neq -1}\log\left(1/ \E\left[\frac{1}{1+N_{-i}} | \phi_i \geq M\right]\right) - 2p \\
		&\geq& H(T) - \sum_{i \neq -1 }\log (1+ \E[N_{-i} \big| \phi_i \geq M]) -2p\\
		&\geq& H(T)/2.
	\end{eqnarray*}

\end{proof}

\section{Information Usage and Classification Overfitting: proof of Prop.~\ref{prop: overfitting}}\label{sec_app:overfitting} 
\begin{proof}
	The empirical $\hat{L}(f)$ and expected $L(f)$ loss of a classifier $f\in \mathcal{F}$ on the training examples $\mathbf{x}\equiv (x_1,...,x_n)$ depend only on the predictions $f({\bf x})\equiv (f(x_1),...,f(x_n))$ it makes on these examples. Let $\mathcal{F}_{\mathbf x} = \{f({\bf x}) : f \in \mathcal{F} \}$ and note that $m=|\mathcal{F}_{\mathbf x}| \leq 2^n$ is finite. Let $f_1,...,f_m$ be functions that make different classifications at $\mathbf{x}$, so $\cup_{1}^m \{f({\bf x})\} = \mathcal{F}_{\bf x}$. 
	
	Now, the overfitting problem studied in Prop.~\ref{prop: overfitting} can be cast in the same framework as the rest of the paper. For each $i\in \{1,...,m\}$, set $\phi_i = \hat{L}(f_i)$ and $\mu_i = L(f_i)$ to be the training error and expected error of classifier  $f_i$. Let $T \in \{1,..,m\}$ be the random index satisfying $\hat{f}({\bf x}) = f_{T}({\bf x})$. Then, our result follows by bounding $|\E[\phi_T-\mu_T]|$. 
	
	If $X\sim {\rm Bern}(p)$ is a Bernoulli random variable with parameter $p$, then $X-p$ is sub-Gaussian with parameter less than 1/4 \cite{buldygin2013sub}. Similarly, if $X_1,..X_n$ are Bernoulli random variables with respective parameters $p_1,...,p_n$, then $n^{-1} \sum_{i=1}^{n} (X_i - p_i)$ is sub-Gaussian with parameter not exceeding $1/4\sqrt{n}$. This immediately implies $\phi_i - \mu_i$ is $\sigma$--sub-Gaussian with $\sigma \leq 1/2n$, so applying Prop.~\ref{prop: main result} implies
	\[
	|\E[\phi_T - \mu_T]| \leq \sqrt{\frac{I(T; \phibf)}{2n}} \leq  \sqrt{\frac{I(T; {\bf Y})}{2n}}.
	\]
	Using the information-processing inequality, and the definition of $T$, we have 
	\[ 
	I(T; \phibf) \leq I(T; {\bf Y}) = I( \hat{f}({\bf x}) , {\bf Y} ),
	\]
	which completes the proof of the first claim. 
	
	The second claim uses a standard link between VC-dimension and the size of $\mathcal{F}_{\bf x}$. Set $S_{\mathcal{F}}(n) = \max\{ |\mathcal{F}_{\bf x}|  : {\bf x} \in \mathcal{X}^{n} \}$ to be the maximum number of ways $n$ points can be classified by the function class. This is often called the growth function of $\mathcal{F}$. We have immediately that $I(T; \phibf) \leq H(T) \leq \log S_{\mathcal{F}}(n)$. Sauer's lemma (see Lemma 1 of \cite{bousquet2004introduction}) shows that 
	\[ 
	S_{\mathcal{F}}(n)    \leq  \begin{cases}
	2^n & \text{if } n < d \\
	\left(\frac{en}{d}\right)^d & \text{if } n \geq d
	\end{cases}
	\]
	where $d<\infty$ is the VC-dimension of $\mathcal{F}$. This implies $\log S_{\mathcal{F}}(n) \leq d \log_{+}\left( \frac{en}{d} \right)$. 
\end{proof}

\section{Bias Control Via FDR Control: proof of Prop.~\ref{prop: FDR}}\label{sec_app:FDR} 
\begin{proof}
	Let $\chi = \mathbf{1}_{\{T \in \Hc_0 \}}$ denote an indicator for a false discovery. Since this is a deterministic function of $T$, 
	\[
	I(T; \phi)= I(T; (\chi ,\phibf)) = I(T; \chi) + I(T; \phibf \mid \chi ) \leq H(\chi) +  I(T; \phibf \mid \chi)
	\]
	Now, using the distribution of $\chi$ we have
	\begin{eqnarray*}
		I(T; \phibf \mid \chi) &\leq&  H(\chi)+ \Prob(\chi=0) I(T; \phibf \mid \chi = 0)\\
		&& + \Prob(\chi=1)I(T; \phibf \mid \chi=1) \\
		&= &  h(\Prob(T\in \Hc_1))+\Prob(T\in \Hc_1)  I(T; \phibf \mid T \in \Hc_1) \\
		&&+ \Prob(T\in \Hc_0)I(T; \phibf \mid T \in \Hc_0)\\
		&\leq&  h(\Prob(T\in \Hc_1))+\Prob(T\in \Hc_1)  I(T; D \mid T \in \Hc_1)\\
		&& + \Prob(T\in \Hc_0)I(T; D \mid T \in \Hc_0)\
	\end{eqnarray*}
	where we have applied chain rule and the data-processing inequality. Now, we have 
	\begin{eqnarray*}
		I(T; D \mid T \in \Hc_1) &=& H(T \mid T \in \Hc_1) - H(T| T \in \Hc_1, \,D)\\ 
		&\leq& \log( \#\Hc_1 ) \\
		&&- \E[\log\left( \#(S_1 \cap \Hc_1)   \right)  \mid T \in \Hc_1] \\
		&=&\log( \#\Hc_1 )\\
		&& - \E[\log\left( (1-\hat{\beta}) \cdot (\#\Hc_1)   \right)  \mid T \in \Hc_1]\\
		&=& -\E\left[ \log\left(1-\hat{\beta}  \right) \mid T \in \Hc_1  \right]\\
		&=& - \log\left(1-\beta\right) + \E\left[\log\left( \frac{1-\beta}{1-\hat{\beta}}  \right) \, \big\vert \, T \in \Hc_1  \right]. 
	\end{eqnarray*}
	Then,
	\begin{eqnarray*}
		\Prob(T \in \Hc_1) I(T; D \mid T \in \Hc_1) &\leq& -\Prob(T\in \Hc_1)\log(1-\beta) \\
		&&+ \E\left[\log\left( \frac{1-\beta}{1-\hat{\beta}}  \right) \mathbf{1}_{\{T \in \Hc_1\} } \right] \\
		&\leq& - \Prob(T\in \Hc_1) \log\left(1-\beta\right) \\
		&&+\E\left[\log_{+}\left( \frac{1-\beta}{1-\hat{\beta}}  \right) \right].
	\end{eqnarray*}
	Essentially the same calculation shows, 
	\[
	\Prob(T \in \Hc_0) I(T; X \mid T \in \Hc_0) \leq  -\Prob(T\in \Hc_0)\log(\alpha) + \E\left[\log_{+}\left( \frac{\alpha}{\hat{\alpha}}  \right) \right].
	\]
	Plugging in $\Prob(T\in \Hc_0)={\rm FDR}$ concludes the proof. 
\end{proof}

\section{Additional Applications of Information Usage}\label{sec_app:otherapplications}
When a data analyst selects hypothesis tests to perform after data exploration, they may  compute extremely small p-values even if there is no signal in the data, and all null hypotheses hold. In this section we apply our information-usage framework to quantify how severely the analyst must explore the data to produce these small p-values.  We also give an illustration of mutual information bound in controlling the value of information in decision-making under uncertainty.


\subsection{The Probability of Small p--values}
Let $\phi_i$ be the observed $p$-value of the $i$th hypothesis and suppose the analyst has to report the p-value $\phi_T$ corresponding to a single hypothesis test from among a large collection of $\phi_1,...,\phi_m$ of observed p-values. Under the null hypothesis, each p-value $\phi_i$ is uniformly distributed, so $\Prob( \phi_i \leq \epsilon) = \epsilon$ for each $\epsilon \in [0,1]$.
Suppose the data analyst rejects the null hypothesis corresponding to $T$ whenever $\phi_T \leq .05$. If $T$ is chosen adaptively so that $\phi_T$ is the smallest p-value among $\phi_1,..\phi_5$, then the probability of falsely rejecting the null hypothesis is $1-(.95)^5 \approx .23$. Therefore, at a significance level of .05, even fairly mild forms of adaptivity can create a substantial risk of false discovery. Nevertheless, we argue in this section that very small p-values are very unlikely unless the mutual information $I( T ; \phibf)$ is large.

To build intuition, imagine that $\phi_1,...,\phi_m \overset{iid}{\sim} {\rm Uniform}(0,1)$. If the  hypothesis $T= \arg\min_{i \leq m} \phi_i$ with the smallest p-value is selected, the reported p-value is expected to be of order $1/m$. In particular, $\E[\phi_T] =1/(m+1)$, and
\[
\Prob\left(\phi_{T} \leq \frac{1}{m}\right) = 1- \left(1-\frac{1}{m}\right)^m \longrightarrow 1-\frac{1}{e}.
\]
Therefore, when selecting among $m \approx e^B$ hypotheses, one expects to observe p-values as small as $\epsilon \approx e^{-B}$ but not smaller. Our next proposition extends this line of reasoning, and replaces $B = \log(m)$ with the mutual information between $T$ and $\phibf$. It shows that when  $\phi_1,...,\phi_m$ are uniformly distributed, but not necessarily independent, one is very unlikely to observe a p value $\phi_T$ much smaller than
$
e^{- I( T ; \phibf)}
$
under an arbitrary adaptive selection procedure $T$.

In fact, the bound provided by the following proposition is stronger. Instead of depending on $I(T; \phibf)$, it depends on the mutual information between $T$ and a more compressed random variable $\mathbf{Z}_{\epsilon}$. Here $Z_{\epsilon, i} \equiv \mathbf{1}(\phi_i < \epsilon)$ and the term $I(T; \mathbf{Z}_{\epsilon})\leq I(T; \phibf)$ is a measure of the dependence of the selection rule on the realization of extremely small p-values.
\begin{prop}\label{prop: p-value}
	Define $
	Z_{\epsilon,i} = \mathbf{1}( \phi_i <\epsilon  )$
	and let $\mathbf{Z}_{\epsilon} = (Z_{\epsilon,1},...,Z_{\epsilon, m})$. If
	$
	\phi_i \sim {\rm Uniform}(0,1)$ for all  $i \in \{1,..,m\} $
	then
	\[
	\Prob(p_T < \epsilon)  \leq \epsilon + \sqrt{\frac{I(T; \mathbf{Z}_{\epsilon})}{\log(1/2\epsilon)}}.
	\]
	
\end{prop}

\begin{proof}[Proof of Proposition \ref{prop: p-value}]
	Since $\phi_i \sim \mbox{Uniform}(0,1)$, $Z_{\epsilon,i} = \mathbf{1}(\phi_i < \epsilon)$ is a Bernoulli random variable with parameter $\epsilon$ and $\E[Z_i] = \epsilon$. We use the fact \cite{buldygin2013sub} that a probability $p$ Bernoulli random variable is sub-Gaussian with parameter
	\[
	\sigma = \sqrt{\frac{1-2p}{2\log((1-p)/p)}} \leq \sqrt{\frac{1}{2\log(1/2p)}}.
	\]
	Combining this with Proposition~\ref{prop: main result}, we have the desired result
	\[
	\E[Z_T] - \E[\mu_T] = \Prob(p_T < \epsilon) - \epsilon \leq \sqrt{\frac{I(T; \mathbf{Z}_{\epsilon})}{\log(1/2\epsilon)}}.
	\]
\end{proof}

To interpret this result, suppose the selection procedure $T$ reports the minimal $p$-value and $\epsilon = 2^{-k}$. If we test $2^{k}$ independent hypotheses, then standard multiple hypotheses testing theory tells us that there is a non-neglible probability that $p_T$ is less than $\epsilon$. This shows up in the bound of Proposition~\ref{prop: p-value} since $\sqrt{\frac{I(T; \mathbf{Z}_{\epsilon})}{\log(1/2\epsilon)}} \approx 1$. However, when there is correlation among the hypotheses, $I(T; \mathbf{Z}_{\epsilon})$ can be significantly less than $2^k$, and our bound quantifies the risk of false discovery in this more nuanced setting.

\subsection{Regret Analysis and the Value of Information}
Consider a general problem of optimization under uncertainty. A decision-maker would like to choose the action $x$ from a finite set $\mathcal{X}$ that solves
$
\max_{x \in \mathcal{X}} f_{\theta}(x).
$
Here $\theta$ is an unknown parameter that is drawn from a prior distribution over a set of possible parameters $\Theta$. We consider the decision-maker's expected shortfall in performance due to not knowing the parameter $\theta$:
\[
\E[\max_{x\in \mathcal{X}}f_{\theta}(x)] - \max_{x\in \mathcal{X}}\E[f_{\theta}(x)].
\]
This measures the value of perfect information about $\theta$: the expected improvement in decision quality that would result from resolving uncertainty about the identity of $\theta$. This is sometimes called the \emph{Bayes risk} or \emph{Bayesian regret} of the decision $\arg\max_{x \in \mathcal{X}}\E[ f_{\theta}(x)]$.

Our main result provides an information theoretic bound on Bayes risk. Let
$X^* \in \arg\max_{x \in \mathcal{X}} f_{\theta}(x) $
denote a true maximizer of the function $f_\theta$. Here $X^*$ is a random variable, since $\theta$ is random, and $X^*$ is a function of $\theta$. Let $\mu(x) = \E[ f_{\theta}(x)]$.
\begin{prop}\label{prop: regret}
	If for each for each $x \in \mathcal{X}$, $f_{\theta}(x) - \mu(x)$ is $\sigma$ sub-Gaussian, then
	\[
	\E[\max_{x\in \mathcal{X}}f_{\theta}(x)] - \max_{x\in \mathcal{X}}\mu(x) \leq \sigma \sqrt{2 H(X^*)}
	\]
\end{prop}
\begin{proof}
	Note that
	\[
	\max_{x \in \mathcal{X}} \mu(x) \geq \E[\mu(X^*)]
	\]
	and
	\[
	\E[\max_{x\in \mathcal{X}}f_{\theta}(x)] = \E[f_{\theta}(X^*)]
	\]
	Therefore,
	\begin{eqnarray*}
		\E[\max_{x\in \mathcal{X}}f_{\theta}(x)] - \max_{x\in \mathcal{X}}\mu(x) &\leq& \E[f_{\theta}(X^*)] - E[\mu(X^*)] \\
		&\leq& \sigma \sqrt{2 I(X^* ; \theta)}\\
		&=&\sigma \sqrt{2 H(X^*)}
	\end{eqnarray*}
	
\end{proof}

\section{Aditional experimental details}\label{sec_app:expt_details}
Here we provide additional details for the LARS bias experiments of Figure~\ref{fig:lars_example}. We consider random design matrix $X \in \mathbb{R}^{100\times 1000}$ whose entries are i.i.d. samples from $\mathcal{N}(0,1)$. The rows of $X$ are then normalized to have unit variance. The effects are represented by the vector $\mathbf{\beta} \in \mathbb{R}^{1000}$.
The first 20 entries of $\beta$ are set to a constant $s$---corresponding to the signals---and rest of the entries are all set to be 0.  By increasing $s$, we increase the signal-to-noise in the data. The \emph{low}, \emph{medium} and \emph{high} signal settings corresponds to setting $s = 0.04, 0.06$ and $0.08$, respectively. Finally the outcomes are given by $y = X\cdot \beta + \epsilon$, where $\epsilon \sim \mathcal{N}(0, I_{100}/10)$ is the noise. We consider the full selection path of LARS on $X$ and $y$. Let the index $T_i$ denote the $i$th feature to enter the subset selected by LARS.

In this experiment, for simplicity, we quantify the bias on the \emph{univariate} regression coefficients. More concretely, suppose we have the true values $y^* = X\cdot \beta$. Then we can use least squares between $y^*$ and the $j$th column of $X$ to determine the true univariate coefficient $\beta^*_j$ of the feature $j$. From the noisy observations $y$, we can similarly compute the noisy univariate coefficient $\hat{\beta}_j$.  We quantify the bias $\hat{\beta}_{T_i}- \beta^*_{T_i}$, for $i = 1, 2, ...$. This bias quantifies how much LARS overfit to the noise in the data. 

\section{Complete Analysis of the Multi-step Data Analysis Model}\label{sec_app:multistep}

\begin{proof}[Proof of Lemma \ref{lem: composition}]
	Since, conditional on $H_{k}$, $T_{k+1}$ is independent of $\phibf$, the data-processing inequality for mutual information implies,
	\[
	I(T_{k+1} ; \phibf) \leq I(H_{k} ; \phibf).
	\]
	Now we have,
	\[
	I( H_k ; \phibf) = \sum_{i=1}^{k} I\left( (T_i, Y_{T_i}) ; \phibf |  H_{i-1} \right).
	\]
	We complete the proof by simplifying the expression for $I\left( (T_i, Y_{T_i}) ; \phibf |  H_{i-1} \right)$. Let $\phibf_{(-i)} = (\phi_{j} : j \neq i)$. Then,
	\begin{eqnarray*}
		I\left( (T_i, Y_{T_i}) ; \phibf |  H_{i-1} \right)&=& I\left( T_i ; \phibf |  H_{i-1} \right)\\
		&&+ I\left( Y_{T_i} ; \phibf |  H_{i-1}, T_i \right)\\
		&=& I\left( Y_{T_i} ; \phibf |  H_{i-1}, T_i \right)\\
		&=& I(Y_{T_i}; \phi_{T_i} | H_{i-1}, T_i )\\
		&&+I(Y_{T_i}; \phi_{(-T_{i})} | H_{i-1}, T_i, \phibf_{T_i}) \\
		&=& I(Y_{T_i}; \phi_{T_i} | H_{i-1}, T_i ),
	\end{eqnarray*}
	where the final equality follows because, conditioned on $\phi_{T_i}$, $Y_{T_i}$ is independent of $\pmb \phi_{(-T_i)}$.
\end{proof}

\begin{proof}[Proof of Lemma \ref{lem: lemma2}]
	\begin{eqnarray*}
		I(X; Y) &=& -\frac{1}{2} \log \left( 1 - \frac{\sigma_1^2}{\sigma_1^2 + \sigma_2^2} \right) \\
		&=& -\frac{1}{2} \log \frac{\sigma_2^2}{\sigma_1^2 + \sigma_2^2}\\
		& =& \frac{1}{2}\log \left(1+ \frac{\sigma_1^2}{\sigma_2^2} \right).
	\end{eqnarray*}
\end{proof}

\begin{lem}\label{lem: variance bound on mutual info}
	Let $X$ be a real value random variable with variance $\sigma_X^2= (X - \E[X])^2]$ and $W \sim \N(0, \sigma^2_W)$ be a normal random variable that is independent of $X$. Then
	\[
	I(X; X+W) \leq \frac{\sigma_X^2}{\sigma_W^2}
	\]
	\begin{proof}
		Let $p_{X}(x)$ denote the density of $X$ with respect to some base measure $\nu$ over $\mathcal{X}$. Then we have
		\begin{eqnarray*}
			&&I(X; X+W)\\&=& \intop_{\mathcal{X}} D(\Prob(x+W=\cdot)\, || \, \Prob(X+W=\cdot))p_{X}(x) d\nu(x) \\
			&\overset{(a)}{\leq} & \intop_{\mathcal{X}} \left[ \intop_{\mathcal{X}} D(\Prob(x_1+W=\cdot)\, || \, \Prob(x_2+W=\cdot))p_{X}(x_2) d\nu(x_2) \right] \\
			&&\times p_{X}(x_1) d\nu(x_1) \\
			&\overset{(b)}{=}& \intop_{\mathcal{X}}\intop_{\mathcal{X}} \frac{(x_1 - x_2)^2}{2\sigma_W^2}p_{X}(x_1)p_{X}(x_2)d\nu(x_1)d\nu(x_2) \\
			&\overset{(c)}{=}& \frac{\sigma_X^2}{\sigma_W^2}.
		\end{eqnarray*}
	\end{proof}
	Here inequality (a) uses the convexity of KL divergence, (b) follows from the formula for the KL divergence between univariate normal distributions $\N(x_1, \sigma^2)$ and $\N(x_2, \sigma^2)$, and (c) uses that if $X_1$ and $X_2$ are iid random variables with mean $\mu$, then
	\[
	\E[(X_1 - X_2)^2] = \E[(X_1 - \mu + \mu - X_2)^2]= 2 \E[(X_1 - \mu)^2].
	\]
\end{lem}

We prove a more general statement of Proposition~\ref{prop: nsquared_short}.

\begin{prop} \label{prop: nsquared}
	Suppose $\phi_i \sim \N(\mu_i, \frac{\sigma^2}{n})$ and $(\phi_1, ..., \phi_k)$ is jointly Gaussian for any $k$. If for the $j$th query, $
	Y_{T_j} = \phi_{T_j} + W_{j}
	$
	where
	$W_j \sim \N(0, \frac{\omega^2_j}{n})$ and $(W_1, W_2, ...)$ is independent of $\phibf$,
	then
	\[
	\E[|Y_{T_{k+1}} - \mu_{T_{k+1}}|] \leq \frac{\sigma}{\sqrt{n}}+ c_1\left( \frac{\omega_{k+1}}{\sqrt{n}}+\sigma^2\sqrt{\frac{\sum_{j=1}^{k} w_j^{-2}}{n}}  \right).
	\]
	If $\omega_j = \sigma j^{1/4}$ for each $j \in \mathbb{N}$, then for every $k\in \mathbb{N}$
	\[
	\E[|Y_{T_{k+1}} - \mu_{T_{k+1}} |] \leq c_2 \left( \frac{\sigma k^{1/4}}{ n^{1/2}}      \right)
	\]
	where $c_1$ and $c_2$ denote universal constants that are independent of $\sigma, \omega, k, $ and $n$.
\end{prop}

\begin{proof}[Proof of Proposition \ref{prop: nsquared}]
	\begin{eqnarray*}
		\E[|Y_{T_{k+1}} - \mu_{T_{k+1}} |]  & \leq & \E[|Y_{T_{k+1}} - \phi_{T_{k+1}} |] +\E[|\phi_{T_{k+1}} - \mu_{T_{k+1}} |] \\
		&\leq &
		\sqrt{\frac{2\omega_{k+1}}{\pi n}}
		+\E[|\phi_{T_{k+1}} - \mu_{T_{k+1}} |] \\
		&\leq & \sqrt{\frac{2\omega_{k+1}}{\pi n}}+ \frac{\sigma}{\sqrt{n}} + c\cdot \sigma \sqrt{\frac{ 2I(T_{k+1}; \phibf)}{n}}
	\end{eqnarray*}
	where $c$ is a universal numerical constant. The second inequality  uses the expected value of the half-normal distribution, and the third inequality follows from Proposition \ref{prop: absolute and squared error}.
	
	The desired result follows by bounding the mutual information term. Applying Lemma \ref{lem: composition},  we have
	\[
	I(T_{k+1} ; \phibf) \leq \sum_{i=1}^{k} I(Y_{T_i} ; \phi_{T_i} | H_{i-1}, T_i )
	\]
	where $
	H_{k} = \left(T_1, Y_{T_1}, T_2, Y_{T_2},..., T_k, Y_{T_k}\right)
	$ denotes the history of interaction up to time $k$. Because the $\phi_i$'s are jointly Gaussian, and observation noise is Gaussian, the posterior $\Prob(\phi_j =\cdot | H_{i-1})$ is Gaussian with conditional variance less than $\sigma^2/n$.  Moreover, conditional on $H_{i=1}$,  $T_i$ is independent of $(\phi_1, \phi_2...)$ and $(Y_1, Y_2, ....)$, so $\phi_{T_i}| H_{i-1}, T_i$ is normally distributed with variance less than $\sigma^2/n$.
	
	Lemma \ref{lem: lemma2} implies
	\[
	I(Y_{T_i} ;  \phi_{T_i} |H_{i-1}, T_i) \leq \frac{\sigma^2/n}{2\omega_{i}^{2}/n} = \frac{\sigma^2}{2\omega^{2}_{i}}
	\]
	and therefore
	\[
	I(T_{k+1}; \phibf) \leq \left(\frac{\sigma^2}{2}\right)\sum_{i=1}^{k} \omega_i^{-2}.
	\]
	Plugging this into the earlier bound implies
	\[
	\E[|Y_{T_{k+1}} - \mu_{T_{k+1}} |]  \leq  \sqrt{\frac{2\omega_{k+1}}{\pi n}}+ \frac{\sigma}{\sqrt{n}} + c\sigma^2 \sqrt{\frac{\sum_{i=1}^{k} \omega_i^{-2} }{n}},
	\]
	which is the desired result.
\end{proof}

\begin{example}[Adaptively fitting a linear model \cite{dwork2014preserving}]\label{eg:overfittinglinear}
	A data-analyst collects $n$ samples of $\theta_1,...\theta_n \overset{iid}{\sim} \mathcal{D} $ of  $k$ dimensional vectors drawn from an unknown distribution and $\hat{\theta} $ is the average of the $\theta_i$'s. She would like to find a unit vector $x$ that is highly correlated with this distribution, in the sense that $\E_{\theta \sim \mathcal{D}}[ x^T \hat{\theta}]$ is large. To do this, she looks to maximize $x^T \hat{\theta}$.
	
	Suppose $\mathcal{D}= \N(0,\sigma^2 I)$, so $\E_{\theta \sim \mathcal{D}} [x^T \theta] = 0$ for all $x$. Nevertheless, the analyst can still find a vector with a large inner product with $\hat{\theta}$. Imagine she collects $k$ measurements of $\hat{\theta}$, allowing her to completely uncover the vector, and then chooses $X =\hat{\theta}/ \|\hat{\theta}\|$ Then, since $\hat{\theta} \sim \N(0,\frac{\sigma^2}{n} I)$,
	\[
	\E[ X^T \hat{\theta}] = \E[\| \hat{\theta}\| = \Theta\left( \sigma\sqrt{\frac{k}{n}}\right).
	\]
\end{example}

\section{Mutual-information vs max-information}\label{sec_app:maxinfo}

Recent work has proposed max-information \cite{dwork2015reusable}, and its generalization, approximate max-information, as a metric to control the error of a worst-case, adversarial, data analyst. This notion was motivated by techniques from differential privacy, which shows that a differentially private mechanism have low approximate max-information, and hence has low error even when the analyst is adversarial.


To understand the relationship between mutual--information and max--information, we revisit the rank selection example from Section \ref{selective}. While max--information provides a powerful tool for analyzing the behavior of a worst-case adaptive protocol, this example shows it can exhibit counter-intuitive behavior when analyzing specific selection procedures.

We assume
\[
\phi_i \sim \begin{cases}
\N(\mu, \sigma^2) & \text{If } i=I^*\\
\N(0, \sigma^2) & \text{If } i\neq I^*
\end{cases}
\]
where $ \mu \geq 0$. The analyst selects $T=\arg\max_{i} \phi_i.$ As discussed in Section \ref{selective}, bias decreases as the signal strength  $\mu$ increases, and this follows transparently from our information theoretic bound. Indeed, as $\mu$ grows $T$ concentrates on $I^*$, and
\[
I(T; \phibf) = H(T) = \sum_{i=1}^{m} \Prob(T=i)\log\left(\frac{1}{\Prob(T=i)}\right)
\]
decreases. This scaling is intuitive. As $T$ concentrates on $I^*$ the selection protocol becomes less and less adaptive, and hence we expect both the selection bias as well as the bias bound which depends on $I(T; \phibf)$ to decrease.

In contrast max-information has the opposite scaling in this setting: it increases as the signal $\mu$ increases and bias decreases. In fact,
\begin{eqnarray*}
	I_{\infty}(T ; \phibf) &=& \max_{i, {\bm y}} \log\left( \frac{\Prob(\phibf = {\bm y}, T=i) )}{\Prob(\phibf = {\bm y})\Prob(T=i)} \right) \\
	&=& \max_{i, {\bm y}} \log\left( \frac{\Prob(T=i | \phibf = {\bm y}) )}{\Prob(T=i)} \right)\\
	& =& \max_{i} \log\left( \frac{1}{\Prob(T=i)}\right),
\end{eqnarray*}

where the maximum is over ${\bm y} \in \mathbb{R}^{m}$ and is attained  for any ${\bm y}$ with $i=\arg\max_{j} y_j$. By symmetry, $I_{\infty}(T ; \phibf) = \log\left( \frac{m-1}{ \Prob(T \neq I^*) } \right)$, which \emph{increases} as the the probability of selecting $I^*$ grows. Therefore, max--information is minimized when the data analyst inappropriately uses rank-selection even though there is no signal in the data ($\mu = 0$). As $\mu$ increases, so the data-analyst detects $I^*$ with probability tending to 1, max-information increases toward infinity.

The related notion of approximate max-information can exhibit similar counter-intuitive behavior. Following \cite{dwork2015generalization}, the approximate max-information at level $\beta$ is defined to be 
\[ 
I^\beta_{\infty}\left( T ; \phibf  \right) := \underset{\substack{\mathcal{O}  \subset [m]\times \mathbb{R}^{m} \\ \Prob((T, \phibf) \in \mathcal{O})\geq \beta } }{\max} \log\left( \frac{\Prob( (T; \phibf) \in \mathcal{O}  )-\beta}{\Prob( (T; \tilde{\phibf}) \in \mathcal{O}) }  \right).
\]

\begin{lem}
	If $T= f(\phibf)$ is a deterministic function of $\phibf$, then
	\[ 
	I^\beta_\infty(T ; \phibf) \geq \underset{\substack{i \leq m \\ \Prob(T=i) \geq 2\beta}  }{\max} \log\left( \frac{1}{\Prob(T=i)} \right) - \log(2)
	\]
	for any $i\in \{1,...m\}$ with $\Prob(T=i) \geq 2\beta$. 
\end{lem}
\begin{proof}
	Let $\tilde{\phibf}$ denote a random variable drawn from the marginal distribution of $\phibf$, but drawn independently of $T$. Define $\Phi_{i} = \{x \in \mathbb{R}^{m} : f(x)=i   \}$ to be the decision region corresponding to element $i.$ Then
	\[
	\Prob(T=i, \phibf \in \Phi_i) = \Prob(T=i) \Prob(\phibf \in \Phi_i | T=i)= \Prob(T=i)
	\]
	whereas
	\[
	\Prob(T=i, \tilde{\phibf} \in \Phi_i) = \Prob(T=i) \Prob(\tilde{\phibf} \in \Phi_i )= \Prob(T=i)^2.
	\]
	
	If $\Prob(T=i) \geq 2\beta$, then $\mathcal{O}= \{(i , x) : x \in \Phi_{i} \}$ is feasible, and therefore
	\begin{eqnarray*}
		I^\beta_{\infty}\left( T ; \phibf  \right) &\geq& \log\left( \frac{\Prob(T=i, \phibf \in \Phi_i)-\beta}{\Prob(T=i, \tilde{\phibf} \in \Phi_i)} \right) \\
		&= & \log\left( \frac{\Prob(T=i)-\beta}{\Prob(T=i)^2} \right) \\
		&\geq & \log\left(\frac{\frac{1}{2}\Prob(T=i)}{\Prob(T=i)^2} \right)\\
		&=&  \log\left( \frac{1}{\Prob(T=i)} \right) - \log(2).
	\end{eqnarray*}
	
\end{proof}

When there is signal in the data, $\Prob(T=i)$ is small for those $\phi_i$ that do not have signal (i.e. a true null). When $\beta$ is sufficiently small so that $\Prob(T=i) \geq 2\beta$, the above lemma shows that $I^\beta_\infty(T ; \phibf)$ can be large, and can increase as $\Prob(T=i)$ deviates farther from the uniform distribution.

\section*{Acknowledgment}

The authors would like to thank John Duchi, Cynthia Dwork, Vitaly Feldman, Aaron Roth, Adam Smith, Thomas Steinke, David Tse and Tsachy Weissman for feedback. 
J.Z. is supported by NSF AF 1763191 and grants from the Chan-Zuckerberg Initiative.

\ifCLASSOPTIONcaptionsoff
  \newpage
\fi



%

\bibliographystyle{IEEEtran}
\bibliography{references}

%

\begin{IEEEbiographynophoto}{Daniel Russo}
Daniel Russo received his PhD from Stanford University in 2015. He was a postdoc at Microsoft Research from 2015 to 2016 and an assistant professor at Northwestern's Kellogg School of Management from 2016 to 2017. He is currently an assistant professor in the Decision, Risk, and Operations Division at Columbia Business School. 
\end{IEEEbiographynophoto}


\begin{IEEEbiographynophoto}{James Zou}
James Zou received his Ph.D. in applied mathematics from Harvard University in 2014. He was a postdoc at Microsoft Research from 2014 to 2016 and is currently an assistant professor at Stanford University. 
\end{IEEEbiographynophoto}




\end{document}